\documentclass[11pt,a4paper]{article}
\usepackage{times,latexsym}
\usepackage{url}
\usepackage[T1]{fontenc}

\usepackage{hyperref}
\usepackage{url}
\usepackage{algorithm}
\usepackage{algorithmic}
\usepackage{microtype}
\usepackage{graphicx}
\usepackage{amsmath, amssymb, amsfonts, amsthm}
\usepackage{booktabs}
\usepackage{arydshln}
\usepackage{adjustbox}
\usepackage[table,dvipsnames]{xcolor}
\usepackage{mathtools}
\usepackage{multirow}

\usepackage[most]{tcolorbox}
\usepackage{booktabs}
\usepackage{array}

\usepackage{microtype}
\usepackage{enumitem}
\usepackage{adjustbox}
\usepackage{inconsolata}

\newtheorem{theorem}{Theorem}
\newtheorem{lemma}{Lemma}
\newtheorem{corollary}{Corollary}

\definecolor{tablerowcolor}{gray}{0.95}

\usepackage{booktabs}
\usepackage{tabularx}
\usepackage{array}
\usepackage{makecell}
\usepackage{ragged2e}
\usepackage{multirow}

\newcolumntype{L}[1]{>{\RaggedRight\arraybackslash}p{#1}}
\newcolumntype{Y}{>{\RaggedRight\arraybackslash}X}

\usepackage[acceptedWithA]{tacl2021v1}

\usepackage{xspace,mfirstuc,tabulary}

\newif\iftaclinstructions
\taclinstructionsfalse 
\iftaclinstructions
\renewcommand{\confidential}{}
\renewcommand{\anonsubtext}{(No author info supplied here, for consistency with
TACL-submission anonymization requirements)}
\newcommand{\instr}
\fi

\iftaclpubformat 

\else

\fi

\title{When Can We Trust the Sparse Lens? A Certification Framework for SAE Faithfulness}

\author{
  Dibyanayan Bandyopadhyay$^\dagger$
  \and
  Asif Ekbal$^\dagger$
  \\
  $^\dagger$Department of Computer Science and Engineering
  \\
  Indian Institute of Technology Patna
  \\
  Patna, Bihar, India
  \\
  \texttt{\{dibyanayan, asif.ekbal\}@gmail.com}
}

\date{}

\begin{document}
\maketitle

\begin{abstract}
    Sparse autoencoders (SAEs) are increasingly used to study internal
representations in language models (LMs), but it remains unclear when
an SAE-based representation reliably preserves the underlying model's
predictive behavior. We introduce a post-hoc certification framework
for this purpose. At a chosen layer, we replace the model's hidden
activation with its pretrained SAE reconstruction and derive a bound on
the expected loss of the original frozen LM through this SAE-based
proxy. The resulting certificate depends on three quantities measured
on held-out data: the proxy risk, the reconstruction loss gap, and support mismatch
between the calibration feature pool and new inputs. We further derive a more conservative exact-$P$ extension whose
proxy-risk concentration term is uniform over all feature pools of
the same size. Empirically, the certificates become non-vacuous for GPT-2 Small,
Gemma-2B, and Llama-3-8B. A detailed layerwise study of
Llama-3-8B shows that later layers are  easier to
certify, driven mainly by improved reconstruction fidelity and weaker
downstream amplification of reconstruction error. GPT-2 Small shows
much weaker layer dependence, indicating that this pattern is not
universal across model--SAE pairs. Overall, the framework provides a
practical way to determine whether an SAE representation preserves
enough of a frozen LM's predictive behavior to support certification
of the original model. Code is available at: \url{https://github.com/newcodevelop/sparse-lens-certification}.
\end{abstract}

\section{Introduction}
\label{sec:intro}

Sparse autoencoders (SAEs)
\citep{cunningham2023sparseautoencodershighlyinterpretable} are
increasingly used to interpret frozen language models (LMs) using
sparse, human-inspectable features. However, there is currently no
mathematical \emph{certificate} for determining whether, for a given
model--SAE pair, the SAE-induced representation reliably preserves the
predictive behavior of the underlying LM. This question is distinct
from whether individual SAE features are interpretable or
human-understandable, which is outside the scope of this work. Instead,
we ask whether the SAE-induced representation retains enough predictive
information about the frozen LM to obtain a non-vacuous post-hoc upper
bound on the LM's population risk, where non-vacuity means that the
bound lies below the loss of an uninformed uniform predictor. If such a bound can be established, we regard the SAE representation as sufficiently informative, in this predictive risk-level sense, to support certification of the LM. This is the notion of predictive faithfulness studied in this paper.


We operationalize this idea by replacing the native hidden activation
of a frozen LM at a chosen layer with its reconstruction produced by a
pretrained SAE, leaving all downstream layers unchanged. The resulting
patched model, consisting of the SAE-reconstructed activation routed
through the original downstream computation of the LM, serves as an
\emph{SAE-based proxy} for the underlying model. Using this construction, we build on recent work on generalization bounds for language models \cite{lotfi2024nonvacuousgeneralizationboundslarge, lotfi2024unlockingtokensdatapoints} and derive an \textit{SAE-mediated} certificate that upper-bounds the population risk of the original frozen LM through its SAE-based proxy. The resulting bound depends on three measurable quantities:
(i) the empirical risk of the SAE proxy, (ii) the approximation error
introduced by the SAE reconstruction (reconstruction loss gap), and (iii) the probability that a
calibration-derived feature pool fails to cover the features required
on evaluation inputs. We additionally consider an optional exact-$P$
extension that accounts for the combinatorial complexity of the
feature-pool identity. Crucially, this decomposition isolates the
conditions under which the SAE-induced representation retains enough
predictive information to yield a non-vacuous certificate for the
underlying LM.


Our use of \emph{faithfulness} is therefore limited to predictive and
behavioral fidelity through the SAE proxy, rather than semantic or causal
interpretability of individual features. We evaluate this notion
empirically across GPT-2 Small, Gemma-2B, and Llama-3-8B, for which the
resulting certificates become non-vacuous at practical sample sizes. We
then present a layerwise case study on Llama-3-8B, where certifiability
varies sharply with patch location: later layers are substantially easier
to certify than early and middle ones. To understand this effect, we
separate local SAE reconstruction quality from downstream error
propagation and find that later-layer proxies exhibit both stronger local
fidelity and weaker error amplification. As a separate qualitative
observation, tighter late-layer certificates are accompanied by SAE
features whose logit-lens verbalizations are more closely aligned with
the model's next-token behavior. Complementary GPT-2 Small results in
Appendix~\ref{app:gpt2_layers} show much weaker layer sensitivity,
indicating that this depth effect is not universal across the
model--SAE pairs we study. We therefore use Llama-3-8B as a case study
of when and why patch location matters.


\paragraph{Contributions.}
Our main contributions are as follows:
\begin{itemize}

    \item We derive an SAE-mediated certificate for the population risk
    of a frozen LM through its SAE-based proxy. The certificate
    depends on the proxy risk, reconstruction loss gap, and support
    mismatch, while an optional exact-$P$ extension additionally accounts
    for the combinatorial complexity of the feature pool.

    \item We show that the resulting certificates become non-vacuous on
    GPT-2 Small, Gemma-2B, and Llama-3-8B at practical evaluation sample
    sizes.

    \item We perform a layerwise and horizon-conditioned case study on
    Llama-3-8B to analyze how patch location affects certification. For
    this model--SAE pair, later layers are substantially easier to certify
    and exhibit both stronger local fidelity and weaker downstream error
    amplification. Complementary GPT-2 Small results in
    Appendix~\ref{app:gpt2_layers} show much weaker depth dependence,
    indicating that this pattern is not universal across model--SAE pairs.

\end{itemize}

\section{Related Work}


\subsection{Sparse autoencoders and interpretable features}
Sparse autoencoders (SAEs) are now a standard tool for probing hidden representations in large language models. They are motivated by the observations of superposition and polysemanticity: a model may represent many more features than there are neurons, and individual neurons may mix unrelated concepts \citep{elhage2022toymodelssuperposition}. SAEs address this by learning an over-complete but sparse feature basis, often yielding features that are easier to interpret than native activations \citep{bricken2023monosemanticity, cunningham2023sparseautoencodershighlyinterpretable}. Rather than using SAEs primarily for qualitative analysis, we use them
to construct a predictive proxy of the frozen LM and, optionally, a
finite class of pool-restricted proxies for the exact-$P$ complexity
analysis.

\subsection{Generalization bounds for large language models}
The empirical success of heavily over-parameterized models has exposed the limitations of classical uniform-convergence intuition \citep{zhang2017understandingdeeplearningrequires, nagarajan2021uniformconvergenceunableexplain}. Recent work has developed non-vacuous bounds for language models using bounded losses, PAC-Bayes or compression-based arguments, and data-aware analyses \citep{dziugaite2017computingnonvacuousgeneralizationbounds, lotfi2024nonvacuousgeneralizationboundslarge, lotfi2024unlockingtokensdatapoints}. Our paper is closest in spirit to this line, but differs in the object
through which certification is performed. Rather than compressing model
weights, we certify the population risk of a frozen LM through an
SAE-based proxy constructed from its internal activations.

\subsection{Compression, description length, and structural explanations}
Compression-based explanations of generalization are closely related to Occam-style bounds and minimum description length principles \citep{RISSANEN1978465, BLUMER1987377, arora2018strongergeneralizationboundsdeep}. The key idea is that
generalization can sometimes be explained by a concise description of
the learned function, even when the overall parameter count is large.
Our work is related to this perspective, but our main goal is different:
we use the SAE to obtain a structured decomposition of the certification
problem. The resulting certificate is not presented as the tightest
possible post-hoc risk bound; rather, it exposes the proxy risk,
reconstruction-loss gap, and support mismatch, with the exact-$P$
feature-pool complexity appearing only in the optional finite-class
extension.

\section{Preliminaries}\label{prelim}

We first define the SAE related notations, then state the post-hoc certification protocol, and finally formalize the proposed bound. 

\subsection{Sparse Autoencoder (SAE)}
We analyze a base LM, denoted as $M$, which maps an input $x$ to a high-dimensional hidden representation $h(x) \in \mathbb{R}^d$ at a specific layer. To interpret this dense representation, we utilize a Sparse Autoencoder (SAE) $S$, consisting of an encoder $S_E: \mathbb{R}^d \rightarrow \mathbb{R}^m$ and a decoder $S_D: \mathbb{R}^m \rightarrow \mathbb{R}^d$, where the dictionary size $m$ is typically much larger than the model width $d$ ($m \gg d$). The SAE decomposes the activation into a sparse set of interpretable features via the following operations:

i) \textit{Encoding:} The dense hidden state is mapped to an SAE feature
vector $a(x):=S_E(h(x))\in\mathbb{R}^m$. ii) \textit{Sparsification:} We apply a common evaluation-time
$\mathrm{TopK}$ operator that retains the $k$ largest feature values
and sets the remaining coordinates to zero, yielding the sparse code
$c(x):=\mathrm{TopK}(a(x))$. The precise encoder quantity supplied to
$\mathrm{TopK}$ is backend-specific and is detailed in Appendix~\ref{exp_setup}.
$\mathrm{TopK}$ is not assumed to be the native coding rule of every
pretrained SAE checkpoint. iii) \textit{Reconstruction:} The sparse code is mapped back to the original activation space to produce the reconstructed hidden state $\hat{h}(x) := S_D(c(x))$.


\paragraph{SAE-based proxy.}
The SAE-based proxy $S\circ M$ is obtained by feeding $\hat{h}(x)$ into the downstream part of $M$
(from the insertion layer onward) to produce a predictive distribution over outputs.
We write $(S\circ M)(x)$ for the resulting predictive distribution.


\subsection{Overview of the Theoretical Approach}
The goal of Section \ref{sec:problem} is to derive a generalization certificate for the base model $M$ using $S\circ M$. 

Our analysis proceeds in two phases:
i) \textbf{Phase 1 (Freezing):} The base model $M$ and the SAE components ($S_E, S_D$) are pre-trained and fixed. For the purpose of our theorem, they are treated as frozen oracles, not as variable hypotheses.
ii) \textbf{Phase 2 (Certification):}
Using a separate calibration stream, we form the feature pool $G^\star$ by taking the union of all SAE features that appear in the Top-$k$ representations. We then define a pool-restricted proxy
$h_{G^\star}$ that uses only features from this pool. The restricted proxy $h_{G^\star}$, together with $S\circ M$, is
then used to derive an upper bound on the population loss of the original language model $M$.





\section{Problem Definition}\label{sec:problem}


\subsection{Risk Formulation}\label{risk-formulation}
Let $\mathcal{X}$ be the input space and $\mathcal{D}$ be an unknown distribution over $\mathcal{X}$. In the language modeling setting, we take a sample to be a token sequence $x = x_{1:T}$. We define the population risk as: $R(M) := \mathbb{E}_{x_{1:T} \sim \mathcal{D}}\big[\ell(M, x_{1:T})\big]$
and, given $N$ i.i.d. samples $\{x^{(i)}_{1:T}\}_{i=1}^N$, the empirical risk is defined as:
$
    \hat{R}(M) := \frac{1}{N}\sum_{i=1}^N \ell\big(M, x^{(i)}_{1:T}\big)
$
Note that the token sequences ($\{x^{(i)}_{1:T}\}_{i=1}^N$) must be i.i.d. samples for the bound to hold. For that, we break the sequences into a contiguous set of tokens and then sample the sequences uniformly randomly from the dataset. This is the approach also used by \citet{lotfi2024nonvacuousgeneralizationboundslarge}.





\subsection{Setup and notation}
\label{sec:setup}

We now formalize how the SAE-based proxy $S\circ M$, together with the
calibration-derived pool $G^\star$, is used to bound the population
risk of the frozen model $M$.

\paragraph{Loss and risk.}

In the language modeling setting, let the model induce next-token probabilities $p_M(\cdot\mid x_{<t})$ over a vocabulary of size $V$. The standard bits-per-dimension (BPD) loss is:
\begin{equation}
    \ell_{\mathrm{bpd}}(M, x_{1:T}) := -\frac{1}{T}\sum_{t=1}^{T} \log p_M(x_t\mid x_{<t})
\end{equation}
Since $\ell_{\mathrm{bpd}}$ is unbounded when $p_M(x_t\mid x_{<t})$ can be arbitrarily small, we use \emph{prediction smoothing}. For a fixed $\alpha \in (0,1)$, as first proposed and defined in \cite{lotfi2024nonvacuousgeneralizationboundslarge}:
\begin{equation}
    \tilde p_M(\cdot\mid x_{<t}) := (1-\alpha) p_M(\cdot\mid x_{<t}) + \alpha/V
\end{equation}
We then define the smoothed BPD loss:
\begin{equation}
    \ell(M, x_{1:T}) := -\frac{1}{T}\sum_{t=1}^{T} \log \tilde p_M(x_t\mid x_{<t})
\end{equation}

This loss is bounded because $\tilde p_M(x_t\mid x_{<t}) \ge \alpha/V$, hence $\log_2(V/\alpha) - \Delta\le \ell(M, x_{1:T}) \le \log_2(V/\alpha) =: B$, where $\Delta = \log_2 (1 + (1 - \alpha)V/\alpha)$. $V$ is the vocabulary size. For a rigorous derivation, check Appendix A.2 of \cite{lotfi2024nonvacuousgeneralizationboundslarge}.



Let $z_i=x^{(i)}_{1:T}$ be an i.i.d. sequence. For a deterministic predictor $f$ (the LM), for $N$ i.i.d. samples $\{z_i\}_{i=1}^N$, define population and empirical risks
\[
\begin{aligned}
R(f) &:= \mathbb{E}_{z\sim\mathcal{D}}[\ell(f;z)],\\
\widehat{R}(f) &:= \frac{1}{N}\sum_{i=1}^N \ell(f;z_i)
\end{aligned}
\]


\paragraph{Reconstruction-loss gap between $M$ and $S\circ M$.}
Define the point-wise reconstruction loss gap
\begin{equation}
\Delta_{\mathrm{loss}}(z) := \big|\ell(M;z)-\ell(S\circ M;z)\big| \in [0,\Delta]
\end{equation}
and its population and empirical means
\begin{equation}
\begin{aligned}
    \epsilon_{\mathrm{loss}} := \mathbb{E}_{z\sim\mathcal{D}}[\Delta_{\mathrm{loss}}(z)];
\\
\widehat{\epsilon}_{\mathrm{loss}} := \frac{1}{N}\sum_{i=1}^N \Delta_{\mathrm{loss}}(z_i)
\end{aligned}
\end{equation}

\subsection{Feature-pool assumption and restricted proxy class}
\label{sec:pool}

For an index set $G\subseteq[m]:=\{1,\dots,m\}$, let $\mathbf{1}_G\in\{0,1\}^m$ denote its indicator vector.
Define the masked activation vector
\begin{equation}
a_G(x) := a(x)\odot \mathbf{1}_G
\end{equation}
$a_G(x)$ means the activation which is restricted to have non-zero dimensions in the indicated position of the index set $G$.
Define the pool-restricted code and reconstruction
\begin{equation}
c_G(x) := \mathrm{TopK}(a_G(x)),
\quad
\widehat{h}_G(x) := S_D(c_G(x))
\end{equation}
Let $h_G$ denote the predictor obtained by feeding $\widehat{h}_G(x)$ into the downstream layers of $M$,
analogously to $S\circ M$.
Thus $\{h_G\}$ is a family of proxy predictors indexed by pools $G$.

\paragraph{Top-$k$ support.}
Let $\mathrm{supp}(v):=\{j: v_j\neq 0\}$.
Define the top-$k$ support event with respect to a pool $G$:
\begin{equation}
E_G(x) := \Big\{\mathrm{supp}\big(\mathrm{TopK}(a(x))\big)\subseteq G\Big\}
\end{equation}

% ...


By the above argument, there exist an integer
$P\in\{k,\dots,m\}$ and a feature pool
$G^\star\subseteq[m]$ with $|G^\star|=P$ such that
\begin{equation}
\Pr_{x\sim\mathcal{D}}\big(E_{G^\star}(x)\big) \ge 1-\eta
\end{equation}
for some $\eta\in[0,1]$.

We define the following restricted hypothesis class:
\begin{equation}
\mathcal{H}_P := \{h_G : G\subseteq[m],\ |G|=P\}
\end{equation}
Then $|\mathcal{H}_P|\le\binom{m}{P}$. 

\paragraph{Direct certification of a fixed model.}
As the smoothed BPD loss is bounded in an interval of width
$\Delta$, namely $[B-\Delta,B]$, any predictor $f$ that is fixed
before the held-out evaluation sample is drawn satisfies the one-sided
Hoeffding bound
\begin{equation}
    R(f)
    \leq
    \widehat R(f)
    +
    \Delta
    \sqrt{
        \frac{\ln(1/\delta)}{2N}
    }
    =:U_f^{\mathrm H}
\end{equation}
with probability at least $1-\delta$.

In particular, the population risk of the frozen language model $M$ can be certified
directly through $U_M^{\mathrm H}$ which requires only empirical measurement of its test loss $\widehat R(M)$. This is the tightest available risk certificate for the LM in our setting.

\paragraph{Why certify through the SAE proxy?}
As stated above, our question is whether the population risk of the
original model $M$ can be certified \emph{through} an SAE-reconstructed
representation. A direct Hoeffding bound on $M$ can certify $M$, but it
bypasses the SAE entirely and therefore says nothing about whether the
SAE representation preserves the model's predictive behavior. Likewise,
a direct Hoeffding bound on the pool-restricted proxy $h_{G^\star}$
certifies only the proxy, not the original model.

To certify $M$ through this proxy, we must therefore account for two
sources of difference: the loss change introduced by SAE reconstruction,
measured by $\epsilon_{\mathrm{loss}}$, and the cases in which the
calibration-derived feature pool does not cover the features required on
a new test input, measured by $\eta$. Since the point-wise reconstruction loss gap is
bounded by $\Delta$, and $h_{G^\star}$ can differ from the unrestricted
SAE proxy $S\circ M$ only when such a support mismatch occurs, we obtain
the following \emph{SAE-mediated risk decomposition}:
\begin{equation}
    R(M)
    \leq
    R(h_{G^\star})
    +
    \epsilon_{\mathrm{loss}}
    +
    \Delta\eta
\end{equation}

The certificate developed below estimates these terms from held-out
empirical data. It is generally looser than direct Hoeffding on $M$, but its
non-vacuity answers a different question: whether the SAE-induced
representation retains sufficient behavioral fidelity to certify
the original frozen model. Appendices \ref{blf} and \ref{sec:lemmas} provide the supporting loss bounds and auxiliary proofs.




\begin{theorem}[Fixed-pool SAE-mediated certificate]
\label{thm:sae_fixed_pool}
Fix $\delta\in(0,1)$. Conditional on the disjoint calibration
stream, let $G^\star$ denote the resulting feature pool and let
$h_{G^\star}$ be the corresponding pool-restricted SAE proxy.
Then, with probability at least $1-\delta$ over $N$ i.i.d.\
held-out evaluation samples,
\begin{align}
R(M)
\leq\;&
\widehat R(h_{G^\star})
+
\widehat{\epsilon}_{\mathrm{loss}}
+
\Delta\widehat{\eta}
\nonumber\\
&+
2\Delta
\sqrt{
    \frac{\ln(3/\delta)}{2N}
}
+
\Delta
\sqrt{
    \frac{\ln(6/\delta)}{2N}
}
\label{eq:sae_fixed_pool}
\end{align}
We denote the right-hand side of
Eq.~\eqref{eq:sae_fixed_pool} by $\boxed{
U_{\mathrm{SAE}}
}$.

\end{theorem}

\begin{proof}
By Lemmas~\ref{lem:decomp} and~\ref{lem:pool_mismatch},
\[
R(M)
\leq
R(h_{G^\star})
+
\epsilon_{\mathrm{loss}}
+
\Delta\eta
\]
Conditional on the calibration stream, $G^\star$ and hence
$h_{G^\star}$ are fixed before the held-out evaluation draw.
Therefore, one-sided Hoeffding concentration gives, with failure
probability $\delta/3$,
\[
R(h_{G^\star})
\leq
\widehat R(h_{G^\star})
+
\Delta
\sqrt{
    \frac{\ln(3/\delta)}{2N}
}
\]
By Lemma~\ref{lem:eta_hoeffding}, also with failure probability
$\delta/3$,
\[
\eta
\leq
\widehat{\eta}
+
\sqrt{
    \frac{\ln(3/\delta)}{2N}
}
\]
Finally, by Lemma~\ref{lem:epsloss_hoeffding}, with failure
probability $\delta/3$,
\[
\epsilon_{\mathrm{loss}}
\leq
\widehat{\epsilon}_{\mathrm{loss}}
+
\Delta
\sqrt{
    \frac{\ln(6/\delta)}{2N}
}
\]
Substituting these three inequalities into the population transfer
bound and applying a union bound gives
Eq.~\eqref{eq:sae_fixed_pool}.
\end{proof}


\begin{corollary}[Exact-$P$ uniform proxy-risk extension]
\label{cor:sae_occam}
Under the conditions of Theorem~\ref{thm:sae_fixed_pool}, let
$P=|G^\star|$ and
\[
\mathcal H_P=\{h_G:G\subseteq[m],\,|G|=P\}
\]
There are $\binom{m}{P}$ possible feature pools of size $P$, and hence
$|\mathcal H_P|\leq\binom{m}{P}$. If, instead of requiring the
proxy-risk bound only for the fixed predictor $h_{G^\star}$, we require
it to hold simultaneously for every predictor $h_G$ induced by a
feature pool of size $P$, then, with probability at least $1-\delta$,
\begin{align}
R(M)\leq\;&
\widehat R(h_{G^\star})
+\widehat{\epsilon}_{\mathrm{loss}}
+\Delta\widehat{\eta}
\nonumber\\
&+\Delta\sqrt{\frac{\ln(3/\delta)}{2N}}
+\Delta\sqrt{\frac{\ln(6/\delta)}{2N}}
\nonumber\\
&+\Delta\sqrt{
\frac{\ln\binom{m}{P}+\ln(3/\delta)}{2N}}
\label{eq:sae_occam}
\end{align}
We denote the right-hand side of Eq.~\eqref{eq:sae_occam} by
$\boxed{U_{\mathrm{SAE+Occam}}}$. It is denoted \(U_{\mathrm{SAE+Occam}}\) because it is the SAE-mediated certificate augmented with the Occam-style finite-class penalty over all exact-\(P\) feature pools.\footnote{$C_P:=\ln\binom{m}{P}$ is called exact-$P$ feature-pool complexity}
\end{corollary}

\begin{proof}
For any fixed $h_G\in\mathcal H_P$, Hoeffding's inequality gives
\[
\Pr\left(
R(h_G)-\widehat R(h_G)>t
\right)
\leq
\exp\left(
-\frac{2Nt^2}{\Delta^2}
\right)
\]
A union bound over all predictors in $\mathcal H_P$ therefore gives
\begin{align}
    \Pr\left(
\exists h_G\in\mathcal H_P:
R(h_G)-\widehat R(h_G)>t
\right) \nonumber\\
\leq 
|\mathcal H_P|
\exp\left(
-\frac{2Nt^2}{\Delta^2}
\right)
\end{align}

Since an exact-$P$ pool is obtained by selecting $P$ feature
identities from an SAE dictionary of size $m$, $|\mathcal H_P|
\le
\binom{m}{P}$.

Setting the right-hand side to $\delta/3$ yields
\[
t
=
\Delta
\sqrt{
\frac{
\ln\binom{m}{P}
+
\ln(3/\delta)
}{2N}
}
\]
Thus, relative to Theorem~\ref{thm:sae_fixed_pool}, the
fixed-predictor concentration term for $h_{G^\star}$ is replaced
by the above finite-class uniform-convergence term. The mismatch
and reconstruction concentration terms are unchanged. Applying the
same $\delta/3$ confidence allocation and union bound therefore gives
Eq.~\eqref{eq:sae_occam}.
\end{proof}

\paragraph{Incremental cost of exact-$P$ uniformity.}
The distinction between the two certificates is explicit:
\begin{align}
U_{\mathrm{SAE+Occam}}-U_{\mathrm{SAE}}
=
\Delta\Bigg[
&
\sqrt{
\frac{
\ln\binom{m}{P}+\ln(3/\delta)
}{2N}}
\nonumber\\
&-
\sqrt{
\frac{\ln(3/\delta)}{2N}}
\Bigg]
\geq 0
\label{eq:occam_increment}
\end{align}
Thus, the exact-$P$ term does not make ordinary held-out
certification tighter. Instead, it measures the extra cost of replacing
the fixed-predictor proxy-risk concentration term with one that is valid
uniformly over all feature pools of size $P$. 


\paragraph{Relation between the certificates.}
The direct Hoeffding bound $U_M^{\mathrm H}$ is the tightest direct
certificate considered in our setting for the already frozen language
model $M$. However, it does not involve the SAE and therefore gives no
information about whether an SAE-reconstructed representation can be
used to certify the model $M$. The purpose of $U_{\mathrm{SAE}}$ is
different: it certifies the same population risk $R(M)$ through the
pool-restricted SAE proxy $h_{G^\star}$, and therefore must also
account for the reconstruction-loss gap
$\epsilon_{\mathrm{loss}}$ and the support-mismatch term
$\Delta\eta$. The certificate \(U_{\mathrm{SAE+Occam}}\) is more conservative because it replaces the fixed-predictor proxy-risk concentration term with one that is valid uniformly over all feature pools of size \(P\). Thus, $U_M^{\mathrm H}
\;\longrightarrow\;
U_{\mathrm{SAE}}
\;\longrightarrow\;
U_{\mathrm{SAE+Occam}}$
moves from direct certification of $M$, to certification of $M$
through the pool-restricted SAE proxy, and finally to the more conservative
exact-$P$ certificate considered in our framework.

\paragraph{Why consider the exact-$P$ extension?}
Conditional on the calibration stream, $M$, the SAE, and the realized
pool $G^\star$ are fixed. Therefore, the exact-$P$ extension is not
statistically necessary for ordinary held-out certification of either
$M$ or $h_{G^\star}$. We include it for two reasons. First, consider
the class of all proxy predictors induced by feature pools of the same
cardinality $P=|G^\star|$:
\begin{equation}
\mathcal H_P
=
\{h_G:G\subseteq[m],\,|G|=P\},
\;
|\mathcal H_P|
\le
\binom{m}{P}
\end{equation}

Hence, $\ln\binom{m}{P}$ measures the logarithm of the number of
possible pool identities over which the proxy-risk concentration term
is made uniform. Second, $\mathcal H_P$ is the most permissive family
possible given only $m$ and $P$.
The resulting certificate $U_{\mathrm{SAE+Occam}}$ can therefore be
viewed as a conservative stress test of the SAE-mediated decomposition:
it asks whether the bound on $R(M)$ remains informative after replacing
the fixed-predictor proxy-risk concentration term with one that is
uniform over all size-$P$ feature pools. If the certificate remains
non-vacuous under this stronger proxy-risk margin, then the SAE-mediated
certificate for the calibration-selected pool remains informative even
after accounting for this additional finite-class penalty. This comparison isolates the effect of uniform proxy-risk control while keeping the calibration-selected pool and its support-mismatch correction fixed.

\subsection{Calibration-Pool Construction}

In our experiments, the feature pool $G^\star$ is constructed from a
calibration stream that is disjoint from the held-out evaluation data.
Given a calibration set $\mathcal D_{\mathrm{cal}}$, we define
\begin{equation}\label{empirical-pooling}
G^\star
=
\bigcup_{x\in\mathcal D_{\mathrm{cal}}}
\operatorname{supp}\!\left(\operatorname{TopK}(a(x))\right),
\end{equation}
where $a(x)$ denotes the SAE feature activations defined above. Its
cardinality is $P = |G^\star|$.

On the held-out evaluation sample $\{x_i\}_{i=1}^N$, we estimate the
support-mismatch probability by $\widehat{\eta}
=
\frac{1}{N}
\sum_{i=1}^{N}
\mathbf 1\!\left\{E_{G^\star}(x_i)^c\right\}$. These quantities are then used directly in
$U_{\mathrm{SAE}}$ and, through
$\ln\binom{m}{P}$, in $U_{\mathrm{SAE+Occam}}$.


\subsection{Discussion and scope}
\label{sec:discussion}

\paragraph{Measured versus certified faithfulness.}
The quantities $\widehat{\epsilon}_{\mathrm{loss}}$ and
$\widehat{\eta}$ are empirical measurements of a particular
model--SAE pair. The former measures the behavioral distortion
introduced by SAE reconstruction, while the latter measures how often
the calibration-derived feature pool fails to cover the SAE support
required on held-out inputs. Our theory does not guarantee that these
quantities are small. Instead, Theorem~\ref{thm:sae_fixed_pool} uses
their held-out estimates, together with the empirical proxy risk, to
obtain a population-level certificate for the original model.

The role of the theorem is therefore to determine whether the
empirically observed reconstruction loss gap, support mismatch, and
proxy risk are small enough to yield a non-vacuous population-level
bound on the original frozen model. If these quantities are large, the
certificate becomes correspondingly loose or vacuous. Thus,
faithfulness is not built into the SAE by construction; rather, the
relevant distortion quantities are measured on held-out data, and the
theorem determines whether those measurements are sufficient to
certify the population risk of the underlying model.

\paragraph{Scope of the framework.}
Our goal is not to explain language-model generalization or obtain the
tightest bound on the frozen model, which can already be certified
directly. Instead, we ask whether the SAE reconstruction preserves
enough predictive behavior to support a non-vacuous population-risk
certificate for the original LM. Here, ``faithfulness'' refers only to
risk-level behavioral fidelity through the SAE proxy, not to semantic or
causal interpretability of individual features.

\begin{figure}[H]
    \centering
    \includegraphics[width=\linewidth]{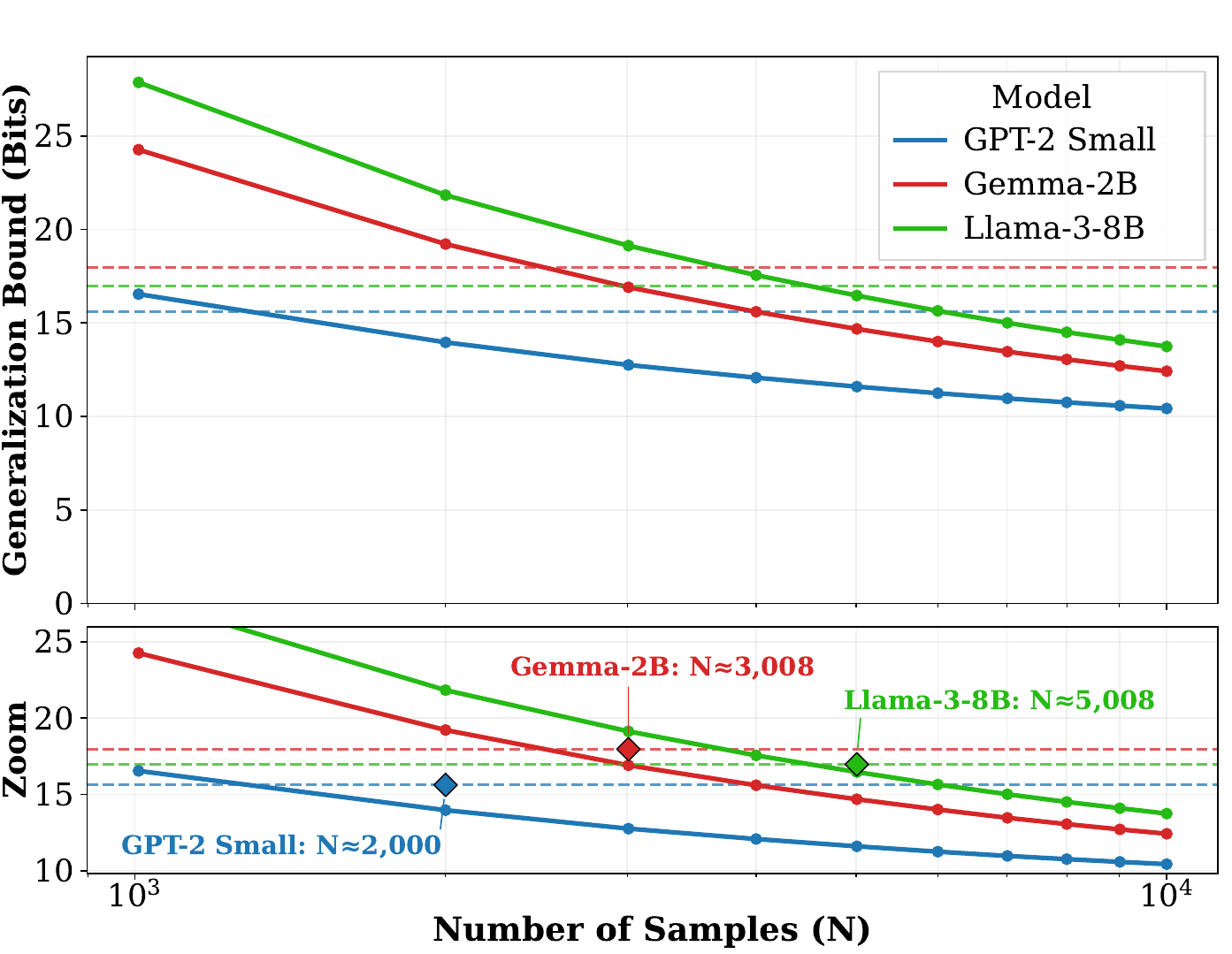}
  \caption{Bound certificate (in bits) against evaluation sample size $N$ for GPT-2 Small (blue) and Gemma-2B (red) and Llama-3-8B. For all experiments $\alpha=0.5$}
    \label{fig:convergence}
\end{figure}


\begin{figure*}[!t]
    \centering
    \includegraphics[width=0.95\textwidth]{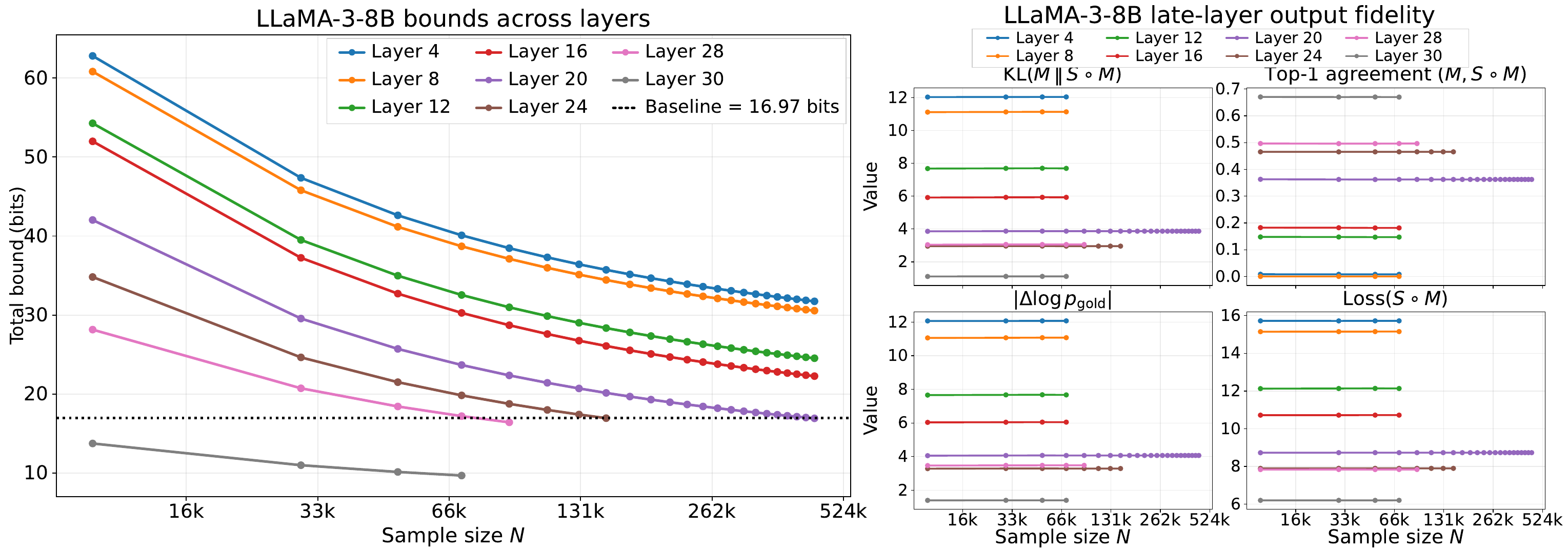}
    \caption{
    Layerwise Llama-3-8B certification and output fidelity.
    Left: certificate curves across patch layers as the evaluation
    sample size $N$ increases; the horizontal dotted line denotes the
    uninformed baseline $R_{\mathrm{unif}}=16.9687$ bits.
    Right: the corresponding unrestricted SAE-proxy output-fidelity trends.}
\label{fig:llama_layerwise_bound_extrap}
\end{figure*}

\begin{table*}[t]
\centering
\setlength{\tabcolsep}{3.2pt}
\scriptsize
\resizebox{\textwidth}{!}{
\begin{tabular}{lrrrrrrrrrrrr}
\toprule
Model
& $m$
& $P$
& $P/m$
& $C_P:=\ln\binom{m}{P}$
& $\widehat R(M)$
& $\widehat R(h_{G^\star})$
& $U_M^{\mathrm H}@N^\star$
& $U_{h_G}^{\mathrm H}@N^\star$
& $U_{\mathrm{SAE}}@N^\star$
& $U_{\mathrm{SAE+Occam}}@N^\star$
& $R_{\mathrm{unif}}$
& $N^\star$ \\
\midrule

GPT-2 Small
& 24,576
& 24,510
& 0.997314
& 453.609
& 7.189
& 7.385
& 7.617
& 7.812
& 9.179
& 13.962
& 15.617
& 2,000 \\

Gemma-2B
& 16,384
& 16,078
& 0.981323
& 1,517.378
& 5.938
& 6.410
& 6.338
& 6.811
& 8.340
& 16.906
& 17.966
& 3,008 \\

Llama-3-8B
& 131,072
& 130,700
& 0.997162
& 2,549.229
& 5.300
& 6.214
& 5.593
& 6.508
& 8.240
& 16.464
& 16.969
& 5,008 \\

\bottomrule
\end{tabular}
}
\caption{Cross-model certification results.
$U_M^{\mathrm H}$ and $U_{h_G}^{\mathrm H}$ are direct
fixed-predictor Hoeffding bounds for the frozen base model and the
pool-restricted SAE proxy, respectively.
All risks and bounds are reported in bits; $\ln\binom{m}{P}$ is
reported in nats.}
\label{tab:main_bounds}
\end{table*}

\section{Experiments}\label{sec:experiments}



Our experiments address three questions: (Q1) Can the SAE-based certificate
become non-vacuous at practical sample sizes? (Q2) How does certifiability
vary with the SAE patch layer? and (Q3) Does feature identity matter beyond
sparsity statistics alone? We evaluate \texttt{GPT-2 Small}
\citep{radford2019language}, \texttt{Gemma-2B}
\citep{gemmateam2024gemma2improvingopen}, and \texttt{Llama-3-8B}
\citep{grattafiori2024llama3herdmodels} with pretrained SAEs evaluated on English C4 \cite{raffel2023exploringlimitstransferlearning}
\footnote{\url{https://huggingface.co/datasets/allenai/c4}}. Full setup
details are in Appendix~\ref{exp_setup}.

\subsection{Non-Vacuous Generalization Bounds}


\paragraph{Measured quantities and evaluation setup.}
For each model--SAE pair, Table~\ref{tab:main_bounds} reports the SAE
dictionary size $m$, the feature-pool coverage $P/m$, the exact
pool-identity term $\ln\binom{m}{P}$, and the direct Hoeffding bounds
for $M$ and $h_{G^\star}$. Here, $P/m$ shows how much of the SAE
dictionary is covered by the calibration-derived pool. Importantly,
$P$ is the size of the union of Top-$k$ supports observed during
calibration, whereas $k$ controls the sparsity of each individual SAE
code; therefore, a large $P$ does not imply that individual codes are
dense. The quantity $\ln\binom{m}{P}$ is the logarithm of the number of
possible feature pools containing exactly $P$ features. For comparison, the direct Hoeffding bounds show the
certificates obtained when $M$ and $h_{G^\star}$ are treated simply as
fixed predictors. To study when the SAE-mediated exact-$P$ certificate becomes
non-vacuous, we increase the evaluation sample size $N$ in
approximately $1{,}000$-sample steps. \textit{We define $N^\star$ as the first evaluated sample size where $U_{\mathrm{SAE+Occam}}<R_{\mathrm{unif}}=\log_2V$. It is only a descriptive first-crossing statistic and carries no simultaneous $1-\delta$ guarantee.}



\paragraph{Results.}
Figure~\ref{fig:convergence} shows how the certificate changes as the evaluation sample
size $N$ increases, while Table~\ref{tab:main_bounds} reports the
corresponding numerical values and model-specific quantities. All three
model--SAE pairs become non-vacuous. The first observed crossing occurs
at $N^\star=2{,}000$ for GPT-2 Small, $3{,}008$ for Gemma-2B, and
$5{,}008$ for Llama-3-8B. At these points,
$U_{\mathrm{SAE+Occam}}$ is $13.962$, $16.906$, and $16.464$ bits,
respectively, below the corresponding uninformed baselines of $15.617$,
$17.966$, and $16.969$ bits. As expected, the direct Hoeffding bounds
reported in Table~\ref{tab:main_bounds} are substantially tighter
because they certify the fixed predictors directly without passing
through the SAE. Among the three configurations, GPT-2 Small requires the fewest evaluation samples to achieve meaningful (non-vacuous) results, while Llama-3-8B requires the most. However, this trend reflects our specific setup rather than proving that larger models are inherently harder to certify. As shown in Appendix \ref{exp_setup} (Table \ref{tab:model_exp_setup}), multiple variables differ across these configurations; including the model family, pretrained SAEs, patch layers, and hookpoints. Because these factors change simultaneously, this comparison cannot isolate the impact of any single variable.



\paragraph{Calibration sensitivity.}
Because $G^\star$ is constructed from calibration data, we examine how
the certificate changes with the calibration-stream size. We use
calibration streams of $0.28$M, $0.56$M, $1.12$M, and $2.24$M tokens,
each formed by taking a larger portion of the same shuffled calibration
stream. For each size, we construct a feature pool and evaluate it on
the same held-out stream ($N=10{,}000$, Top-$k=64$). As shown in
Table~\ref{tab:calibration_sensitivity}, GPT-2 Small and Gemma-2B are
largely stable across these budgets. For Llama-3-8B, increasing the
calibration data reduces $\widehat{\eta}$ from $0.2466$ to $0.0029$
and $U_{\mathrm{SAE}}$ from $12.054$ to $7.919$ bits, while the proxy
risk ($\approx6.207$ bits) and reconstruction-loss gap
($\approx0.915$ bits) remain nearly unchanged. Thus, the main effect of
additional calibration data is better coverage of the features required
on held-out inputs. For $U_{\mathrm{SAE+Occam}}$, the corresponding
change in $\ln\binom{m}{P}$ has an additional finite-sample effect. Results are shown in Table \ref{tab:calibration_sensitivity}. 
Additional sensitivity results appear in Figure \ref{fig:topk_sensitivity} (Appendix \ref{topk-sensitivity}; Top-\(k\)), Figure \ref{fig:bound_decomposition} (Appendix \ref{iid-vs-noise}; input corruption), and Table \ref{tab:frequency_threshold} (Appendix \ref{freq-pool}; frequency-thresholded pools). 


\begin{table}[H]
\centering
\setlength{\tabcolsep}{4pt}
\small
\resizebox{\columnwidth}{!}{
\begin{tabular}{lrrrrrr}
\toprule
Model &
Calibration tokens &
$P/m$ &
$\widehat{\eta}$ &
$\ln\binom{m}{P}$ &
$U_{\mathrm{SAE}}$ &
$U_{\mathrm{SAE+Occam}}$ \\
\midrule
GPT-2 Small
& 0.28M & 0.9940 & 0.0055 & 895.7  & 8.349 & 11.438 \\
& 0.56M & 0.9953 & 0.0025 & 728.4  & 8.302 & 11.067 \\
& 1.12M & 0.9967 & 0.0015 & 546.4  & 8.287 & 10.654 \\
& 2.24M & 0.9973 & 0.0013 & 453.6  & 8.283 & 10.423 \\
\midrule
Gemma-2B
& 0.28M & 0.9771 & 0.0034 & 1783.2 & 7.773 & 12.887 \\
& 0.56M & 0.9784 & 0.0010 & 1703.8 & 7.730 & 12.723 \\
& 1.12M & 0.9799 & 0.0003 & 1607.6 & 7.717 & 12.560 \\
& 2.24M & 0.9813 & 0.0003 & 1517.4 & 7.717 & 12.415 \\
\midrule
Llama-3-8B
& 0.28M & 0.9775 & 0.2466 & 14103.8 & 12.054 & 26.063 \\
& 0.56M & 0.9899 & 0.0666 & 7382.9  & 9.000  & 19.070 \\
& 1.12M & 0.9948 & 0.0179 & 4288.6  & 8.174  & 15.792 \\
& 2.24M & 0.9972 & 0.0029 & 2549.2  & 7.919  & 13.739 \\
\bottomrule
\end{tabular}
}
\caption{Calibration-size sensitivity on a fixed held-out stream
($N=10{,}000$, Top-$k=64$); bounds are in bits and
$\ln\binom{m}{P}$ in nats.
}
\label{tab:calibration_sensitivity}
\end{table}

\begin{table*}[t]
\centering
\setlength{\tabcolsep}{2.9pt}
\scriptsize
\resizebox{\textwidth}{!}{
\begin{tabular}{crrrrrrrrrrrr}
\toprule
Layer
& $P$
& $P/m$
& $C_P:=\ln\binom{m}{P}$
& $\widehat R(h_{G^\star})$
& $\widehat\epsilon_{\mathrm{loss}}$
& $\widehat\eta$
& $U_M^{\mathrm H}@N^\star$
& $U_{h_G}^{\mathrm H}@N^\star$
& $U_{\mathrm{SAE}}@N^\star$
& $U_{\mathrm{SAE+Occam}}@N^\star$
& $N^\star$
& Asymptotic floor \\
\midrule

4
& 72,752
& 0.555054
& 90,049.927
& 15.713
& 10.422
& 0.0090
& --
& --
& --
& --
& Never
& 26.289 \\

8
& 84,117
& 0.641762
& 85,505.007
& 15.147
& 9.856
& 0.0142
& --
& --
& --
& --
& Never
& 25.244 \\

12
& 88,634
& 0.676224
& 82,527.783
& 12.132
& 6.842
& 0.0201
& --
& --
& --
& --
& Never
& 19.316 \\

16
& 88,636
& 0.676239
& 82,526.310
& 10.722
& 5.448
& 0.0518
& --
& --
& --
& --
& Never
& 17.049 \\

20
& 109,646
& 0.836533
& 58,370.397
& 8.734
& 3.450
& 0.0201
& 5.322
& 8.764
& 12.635
& 16.921
& 450k
& 12.524 \\

24
& 119,707
& 0.913292
& 38,641.639
& 7.902
& 2.612
& 0.0122
& 5.344
& 7.955
& 10.915
& 16.942
& 150k
& 10.722 \\

28
& 126,350
& 0.963974
& 20,324.407
& 7.840
& 2.560
& 0.0085
& 5.360
& 7.910
& 10.794
& 16.416
& 90k
& 10.545 \\

30
& 130,700
& 0.997162
& 2,549.229
& 6.206
& 0.915
& 0.0029
& 5.499
& 6.414
& 7.919
& 13.739
& 10k
& 7.171 \\

\bottomrule
\end{tabular}
}
\caption{Layerwise decomposition of the Llama-3-8B certificate.
The layerwise analysis evaluates $N$ at $20$k intervals starting from
$10$k, unlike the finer approximately $1$k-sample intervals used in
the cross-model analysis. $N^\star$ denotes the first evaluated
sample size at which the certificate ($U_{\mathrm{SAE+Occam}}$) becomes non-vacuous. $R_{\mathrm{unif}}=16.9687$ bits.
}
\label{tab:llama_layerwise_bound_extrap}
\end{table*}

\begin{table*}[t]
\centering
\setlength{\tabcolsep}{5pt}
\small
\resizebox{\textwidth}{!}{
\begin{tabular}{c r c c c c c}
\toprule
Layer &
$P$ &
$\mathrm{KL}(M\,\|\,S\circ M)$ &
Top-1 Agree.$(M,S\circ M)$ &
$|\Delta \log_2 p_{\mathrm{gold}}|$ &
Loss$(M)$ &
Loss$(S\circ M)$ \\
\midrule
4  & 72,752  & 12.051 & 0.008 & 12.078 & 5.289 & 15.718 \\
8  & 84,117  & 11.142 & 0.000 & 11.077 & 5.289 & 15.150 \\
12 & 88,634  &  7.696 & 0.147 &  7.672 & 5.289 & 12.134 \\
16 & 88,636  &  5.931 & 0.181 &  6.050 & 5.289 & 10.723 \\
20 & 109,646 &  3.868 & 0.362 &  4.067 & 5.289 &  8.733 \\
24 & 119,707 &  2.957 & 0.466 &  3.290 & 5.289 &  7.902 \\
28 & 126,350 &  3.055 & 0.496 &  3.479 & 5.289 &  7.841 \\
30 & 130,700 &  1.112 & 0.670 &  1.398 & 5.289 &  6.208 \\
\bottomrule
\end{tabular}
}
\caption{
Layer-wise output fidelity of the unrestricted SAE proxy
$S\circ M$ on Llama-3-8B. We report the calibration-pool size $P$,
KL divergence, top-1 agreement, absolute gold-token log-probability
difference, and smoothed losses. KL divergence,
$|\Delta\log_2 p_{\mathrm{gold}}|$, and losses are reported in bits.
}
\label{tab:llama_layerwise_fidelity}
\end{table*}

\subsection{Layerwise certifiability: A case study on Llama-3-8B}
\label{subsec:llama_depth_cert}


We next ask whether the choice of patch layer materially affects certification in a model where layerwise variation is substantial. Llama-3-8B is a useful case study for this purpose because, unlike GPT-2 Small (Appendix \ref{app:gpt2_layers}), it exhibits strong heterogeneity across patch locations under the same certification protocol. We therefore use Llama-3-8B to analyze when layer choice matters, and why. Concretely, we evaluate pretrained SAEs at layers 
$\{4,8,12,16,20,24,28,30\}$ and measure the empirical proxy risk, reconstruction-loss gap, support-mismatch rate, active pool size, and resulting certificate at each layer.

\paragraph{Why later layers are easier to certify.}
Table~\ref{tab:llama_layerwise_bound_extrap} shows two changes as the SAE is inserted at
later layers. First, the \textit{exact-$P$ feature-pool complexity}, $C_P:=\ln\binom{m}{P}$, decreases substantially. Although the calibration-pool size $P$
increases from $72{,}752$ at Layer~4 to $130{,}700$ at Layer~30,
all evaluated pools satisfy $P>m/2$, with $m=131{,}072$. Since $\binom{m}{P}=\binom{m}{m-P}$,
moving $P$ closer to $m$ leaves fewer possible feature pools of size
$P$. Accordingly, $C_P$ decreases from $90{,}049.93$ at Layer~4 to
$2{,}549.23$ at Layer~30.

Second, the SAE proxy itself becomes substantially closer to the
original model. The empirical proxy risk
$\widehat R(h_{G^\star})$ decreases from $15.71$ to $6.21$ bits, while
the reconstruction-loss gap $\widehat{\epsilon}_{\mathrm{loss}}$ decreases
from $10.42$ to $0.92$ bits. Thus, later layers benefit from both a
smaller exact-$P$ feature-pool complexity and substantially better
agreement between the SAE proxy and the original model.

\paragraph{Which factor drives non-vacuity?}
We define the \emph{asymptotic floor} by holding the measured empirical
terms fixed and letting the explicit finite-sample correction terms
vanish as $N\rightarrow\infty$. This isolates the contribution of the
proxy risk, reconstruction-loss gap, and support mismatch from the
finite-sample terms in the certificate. For Layers~4, 8, 12, and 16, the asymptotic floors are
$26.29$, $25.24$, $19.32$, and $17.05$ bits, all above
$R_{\mathrm{unif}}=16.97$ bits. These layers therefore cannot become
non-vacuous even with arbitrarily large evaluation samples. In
contrast, the floors for Layers~20, 24, 28, and 30 are $12.52$,
$10.72$, $10.55$, and $7.17$ bits, respectively, so non-vacuity is
possible. Their first observed non-vacuous sample sizes are $450$k,
$150$k, $90$k, and $10$k. These results show that the main reason later
layers become certifiable is the strong reduction in proxy risk and
reconstruction loss, while the smaller exact-$P$ complexity ($C_P$) provides an
additional finite-sample advantage by reducing the amount of
evaluation data required.

To understand this layer-wise effect more directly, we also measure \emph{output fidelity} between the original model \(M\), the unrestricted SAE proxy \(S \circ M\). The results are reported in Table~\ref{tab:llama_layerwise_fidelity} and Figure~\ref{fig:llama_layerwise_bound_extrap} (Right) and they visualize the corresponding late-layer fidelity trends. The same overall trend appears across all metrics:
later-layer reconstructions preserve the base model's output much more
closely. From Layer~4 to Layer~30, the KL divergence decreases from
$12.051$ to $1.112$ bits, top-1 agreement increases from $0.008$ to
$0.670$, and the absolute gold-token log-probability difference
decreases from $12.078$ to $1.398$ bits. Consistently, the loss of
$S\circ M$ decreases from $15.718$ to $6.208$ bits, approaching the
base-model loss of $5.289$ bits. Although the improvement is not
strictly monotonic at every intermediate layer, the overall trend is
clear: SAE reconstruction at later layers provides a substantially
closer behavioral proxy for the original model. This supports the
previous finding that the improved certifiability of later layers is
driven primarily by reduced proxy and reconstruction error, with the
smaller exact-$P$ complexity providing an additional finite-sample
advantage.

\paragraph{Disentangling Representation Quality from Error Propagation.}
While our layerwise results demonstrate that later-layer SAE proxies are closer to the original model, in the sense that they yield tighter bounds and smaller proxy distortion, this alone does not identify the root cause of the loose bound seen for the earlier layers. A looser certificate at an early layer could reflect two distinct phenomena: either the local SAE representation is intrinsically insufficient, or a minor local reconstruction error is being drastically amplified by the deep downstream computation. To determine whether later layers have better local proxy fidelity or
merely benefit from a shorter path to the final output, we introduce a
horizon-conditioned rollout analysis.

Let $F_{\ell \rightarrow \ell+h}$ denote the subnetwork mapping the residual activation at layer $\ell$ to the activation $h$ blocks later.\footnote{When $\ell+h$ exceeds the network depth, we use the full remaining computation to the final next-token distribution.} For a patch layer $\ell$ and rollout horizon $h \geq 0$, we compare the intermediate activation distribution at layer $\ell+h$, obtained by feeding the base activation $h_\ell(x)$ versus the SAE reconstruction $\hat h_\ell(x)$ through $F_{\ell \rightarrow \ell+h}$. We measure the divergence:$$\mathrm{KL}^{(\ell,h)}(x)=\mathrm{KL}\!\left(p_{\mathrm{base}}^{(\ell,h)}(\cdot \mid x)\,\|\,p_{\mathrm{proxy}}^{(\ell,h)}(\cdot \mid x)\right)$$This isolates the two effects: matched short horizons ($h \in \{0,1\}$) probe the local representational fidelity of the SAE patch, while the growth of $\mathrm{KL}^{(\ell,h)}$ as $h$ increases quantifies how strongly that local mismatch is amplified by downstream computation.

\begin{figure}[h]
    \centering
    \includegraphics[width=\columnwidth]{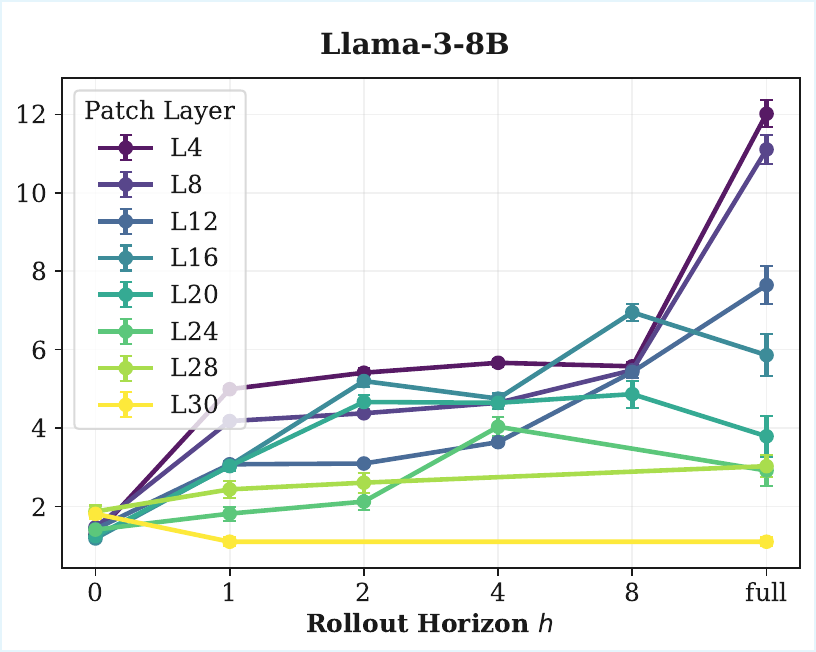}
    \caption{Horizon-conditioned Base-Proxy KL: Local Fidelity vs Downstream Error Accumulation}
    \label{fig:horizon-conditioned}
\end{figure}


\definecolor{QBlueLite}{HTML}{EAF2FA}
\definecolor{QBorder}{HTML}{D5DCE5}
\definecolor{QGray}{HTML}{5B6573}
\definecolor{QPrompt}{HTML}{F6F8FB}
\definecolor{QGreen}{HTML}{1E8E5A}
\definecolor{QRed}{HTML}{C43D2B}

\newtcolorbox{qualpanel}[2][]{%
  enhanced,
  colback=white,
  colframe=QBorder,
  boxrule=0.6pt,
  arc=1.5mm,
  left=1.5mm,right=1.5mm,top=1.5mm,bottom=1.5mm,
  title={#2},
  colbacktitle=QBlueLite,
  coltitle=black,
  fonttitle=\bfseries\footnotesize,
  attach boxed title to top left={xshift=1mm,yshift=-1mm},
  boxed title style={arc=1mm,boxrule=0pt,colback=QBlueLite,left=1.2mm,right=1.2mm,top=0.5mm,bottom=0.5mm},
  #1
}

\begin{figure*}[t]
\centering
\small

\begin{tcolorbox}[
  enhanced,
  colback=white,
  colframe=QBorder,
  boxrule=0.8pt,
  arc=2mm,
  left=1.5mm,right=1.5mm,top=1.5mm,bottom=1.5mm,
  width=\textwidth,
  title={\textbf{Qualitative case study: Later-layer SAE concepts}},
  colbacktitle=white,
  coltitle=black,
  fonttitle=\bfseries\large,
  boxed title style={boxrule=0pt,colback=white}
]

\begin{minipage}[t]{0.485\textwidth}
\begin{qualpanel}{Prompt A: Deductive inheritance \hfill \textcolor{QGray}{Layer 12}}
\textbf{Prompt:} {\scriptsize\ttfamily\vspace{0.5mm}
\begin{tcolorbox}[colback=QPrompt,colframe=QBorder,boxrule=0.4pt,arc=1mm,left=1mm,right=1mm,top=0.5mm,bottom=0.5mm]
Premises: Every violinist is a musician. No musician is completely silent on stage. Mira is a violinist. Therefore, Mira is not completely
\end{tcolorbox}}

\vspace{0.5mm}
{\scriptsize \textbf{Prediction:} \textbf{Gold:} \texttt{' silent'} \hspace{0.5mm}|\hspace{0.5mm} \textbf{Base:} \texttt{' silent'} \hspace{0.5mm}|\hspace{0.5mm} \textbf{Proxy:} \texttt{' !'} \hspace{0.5mm}|\hspace{0.5mm} \textbf{KL:} 8.12}

\vspace{1.5mm}
\textbf{Top 3 SAE concepts:}\par\vspace{0.5mm}
{\scriptsize
\begin{tabular}{@{} l l l >{\raggedright\arraybackslash}p{3.2cm} @{}}
\textbf{F3807} & \textcolor{QGreen}{\textbf{supports}} & $a=0.59$ & \textcolor{QGray}{\texttt{isses}, \texttt{/full}, \texttt{ARRIER}} \\
\textbf{F99917} & \textcolor{QGreen}{\textbf{supports}} & $a=0.34$ & \textcolor{QGray}{\texttt{stics}, \texttt{references}, \texttt{sip}} \\
\textbf{F65022} & \textcolor{QGreen}{\textbf{supports}} & $a=0.40$ & \textcolor{QGray}{\texttt{anymore}, \texttt{necessarily}, \texttt{yet}} \\
\end{tabular}
}
\end{qualpanel}
\end{minipage}%
\hfill
\begin{minipage}[t]{0.485\textwidth}
\begin{qualpanel}{Prompt A: Deductive inheritance \hfill \textcolor{QGray}{Layer 24}}
\textbf{Prompt:} {\scriptsize\ttfamily\vspace{0.5mm}
\begin{tcolorbox}[colback=QPrompt,colframe=QBorder,boxrule=0.4pt,arc=1mm,left=1mm,right=1mm,top=0.5mm,bottom=0.5mm]
Premises: Every violinist is a musician. No musician is completely silent on stage. Mira is a violinist. Therefore, Mira is not completely
\end{tcolorbox}}

\vspace{0.5mm}
{\scriptsize \textbf{Prediction:} \textbf{Gold:} \texttt{' silent'} \hspace{0.5mm}|\hspace{0.5mm} \textbf{Base:} \texttt{' silent'} \hspace{0.5mm}|\hspace{0.5mm} \textbf{Proxy:} \texttt{' silent'} \hspace{0.5mm}|\hspace{0.5mm} \textbf{KL:} 0.01}

\vspace{1.5mm}
\textbf{Top 3 SAE concepts:}\par\vspace{0.5mm}
{\scriptsize
\begin{tabular}{@{} l l l >{\raggedright\arraybackslash}p{3.2cm} @{}}
\textbf{F36987} & \textcolor{QGreen}{\textbf{supports}} & $a=3.08$ & \textcolor{QGray}{\texttt{silent}, \texttt{silence}, \texttt{quiet}} \\
\textbf{F124438} & \textcolor{QGreen}{\textbf{supports}} & $a=1.38$ & \textcolor{QGray}{\texttt{stateless}, \texttt{ainless}, \texttt{odor}} \\
\textbf{F71772} & \textcolor{QGreen}{\textbf{supports}} & $a=2.33$ & \textcolor{QGray}{\texttt{itude}, \texttt{mente}, \texttt{strangers}} \\
\end{tabular}
}
\end{qualpanel}
\end{minipage}

\vspace{2mm}

\begin{minipage}[t]{0.485\textwidth}
\begin{qualpanel}{Prompt B: Domain constraint \hfill \textcolor{QGray}{Layer 12}}
\textbf{Prompt:} {\scriptsize\ttfamily\vspace{0.5mm}
\begin{tcolorbox}[colback=QPrompt,colframe=QBorder,boxrule=0.4pt,arc=1mm,left=1mm,right=1mm,top=0.5mm,bottom=0.5mm]
Premises: Every enzyme in this tray is a protein. No protein in this tray is insoluble. Catalase is an enzyme in this tray. Therefore, Catalase is not
\end{tcolorbox}}

\vspace{0.5mm}
{\scriptsize \textbf{Prediction:} \textbf{Gold:} \texttt{' insol'} \hspace{0.5mm}|\hspace{0.5mm} \textbf{Base:} \texttt{' insol'} \hspace{0.5mm}|\hspace{0.5mm} \textbf{Proxy:} \texttt{' !'} \hspace{0.5mm}|\hspace{0.5mm} \textbf{KL:} 7.61}

\vspace{1.5mm}
\textbf{Top 3 SAE concepts:}\par\vspace{0.5mm}
{\scriptsize
\begin{tabular}{@{} l l l >{\raggedright\arraybackslash}p{3.2cm} @{}}
\textbf{F106873} & \textcolor{QGreen}{\textbf{supports}} & $a=0.53$ & \textcolor{QGray}{\texttt{must}, \texttt{adol}, \texttt{eroon}} \\
\textbf{F114367} & \textcolor{QGreen}{\textbf{supports}} & $a=0.63$ & \textcolor{QGray}{\texttt{oriously}, \texttt{ched}, \texttt{withstanding}} \\
\textbf{F65022} & \textcolor{QGreen}{\textbf{supports}} & $a=0.40$ & \textcolor{QGray}{\texttt{anymore}, \texttt{necessarily}, \texttt{yet}} \\
\end{tabular}
}
\end{qualpanel}
\end{minipage}%
\hfill
\begin{minipage}[t]{0.485\textwidth}
\begin{qualpanel}{Prompt B: Domain constraint \hfill \textcolor{QGray}{Layer 24}}
\textbf{Prompt:} {\scriptsize\ttfamily\vspace{0.5mm}
\begin{tcolorbox}[colback=QPrompt,colframe=QBorder,boxrule=0.4pt,arc=1mm,left=1mm,right=1mm,top=0.5mm,bottom=0.5mm]
Premises: Every enzyme in this tray is a protein. No protein in this tray is insoluble. Catalase is an enzyme in this tray. Therefore, Catalase is not
\end{tcolorbox}}

\vspace{0.5mm}
{\scriptsize \textbf{Prediction:} \textbf{Gold:} \texttt{' insol'} \hspace{0.5mm}|\hspace{0.5mm} \textbf{Base:} \texttt{' insol'} \hspace{0.5mm}|\hspace{0.5mm} \textbf{Proxy:} \texttt{' a'} \hspace{0.5mm}|\hspace{0.5mm} \textbf{KL:} 1.47}

\vspace{1.5mm}
\textbf{Top 3 SAE concepts:}\par\vspace{0.5mm}
{\scriptsize
\begin{tabular}{@{} l l l >{\raggedright\arraybackslash}p{3.2cm} @{}}
\textbf{F20271} & \textcolor{QGreen}{\textbf{supports}} & $a=1.05$ & \textcolor{QGray}{\texttt{sol}, \texttt{dissolve}, \texttt{soluble}} \\
\textbf{F103948} & \textcolor{QGreen}{\textbf{supports}} & $a=1.32$ & \textcolor{QGray}{\texttt{suspended}, \texttt{traces}, \texttt{content}} \\
\textbf{F119119} & \textcolor{QGreen}{\textbf{supports}} & $a=0.60$ & \textcolor{QGray}{\texttt{active}, \texttt{\_active}, \texttt{\_\_active}} \\
\end{tabular}
}
\end{qualpanel}
\end{minipage}

\end{tcolorbox}
\vspace{-2mm}
\caption{Qualitative comparison of early and late SAE layers on two deduction prompts. We show (i) the prompt, (ii) the base--proxy agreement summary through KL, and (iii) the top active SAE concepts verbalized using feature-level logit-lens tokens. Lower KL and better token verbalizations at later layers indicate a high fidelity SAE reconstruction. Note, all gold/base/proxy next-token strings are
tokenizer subword pieces rather than necessarily whole words.}
\label{fig:qualitative_late_layer_faithfulness}
\end{figure*}
\paragraph{Results: The Dual Advantage of Later Layers.}Figure \ref{fig:horizon-conditioned} reveals three distinct regimes for Llama-3-8B:

i) Early layers ($\ell=4, 8$): Exhibit nontrivial local distortion that undergoes severe amplification even after 1–2 rollout steps, ultimately reaching the largest full-horizon divergence ($>11$ bits).

ii) Middle layers ($\ell=12, 16, 20$): Show lower short-horizon divergence, but still suffer from substantial error accumulation over longer horizons.

iii) Late layers ($\ell=24, 28, 30$): Demonstrate both high local fidelity (low initial KL for $h \in \{0,1,2\}$) and remarkable stability, with divergence remaining flat or growing minimally.


Crucially, this shows the late-layer advantage is not merely a byproduct of a shorter computational path. If downstream length were the sole factor, short-horizon divergences would align across all layers. These results support a dual explanation for the late-layer advantage: later layers exhibit both better local proxy fidelity and weaker downstream amplification of reconstruction error. 

For completeness, the corresponding GPT-2 Small layerwise certificate curves are reported in Figure \ref{fig:gpt-layerwise} (Appendix \ref{app:gpt2_layers}); unlike Llama-3-8B, they remain tightly clustered (for $\mathrm{U}_{SAE}$) across depth, which entails that the layerwise performance pattern is not a universal property of every model but depends largely on how the model and SAE are trained. The contrast between Llama and GPT-2 illustrates the intended role of the framework: to diagnose when patch location is consequential and when it is not.

\paragraph{Qualitative Case Study}
\label{sec:qualitative_case_study}


We complement the dataset-level certification analysis with a qualitative case study. Specifically, we investigate whether the tighter bounds achieved
at later layers are accompanied by SAE features that are more  aligned with the underlying language model's next-token output.


\paragraph{Question and Setup.}
Let $M$ be a frozen base model and $h_\ell(x)$ be its residual activation at layer $\ell$. We construct a \emph{SAE based proxy} by replacing $h_\ell(x)$ with its SAE reconstruction, $\hat h_\ell(x)$, and executing the remaining forward pass. For a given prompt, we compare the original model $M$ against this SAE based proxy at two distinct layers (one early and one late layer). At each layer, we evaluate two metrics: the KL divergence between their next-token distributions, and a logit-lens \citep{wang2025logitlens4llmsextendinglogitlens} verbalization of the most active SAE features.


\paragraph{Layer-wise KL}
For a prompt $x$, we measure proxy drift via the divergence, $\mathrm{KL}(p_M(\cdot \mid x) \,\|\, p_{S \circ M}(\cdot \mid x))$ between the next-token distributions of the base model $M$ and the SAE-patched proxy at layer $\ell$. Lower values indicate that the SAE reconstruction accurately preserves the model's predictive behavior.

\begin{figure*}[t]
    \centering
    \includegraphics[width=1.0\textwidth]{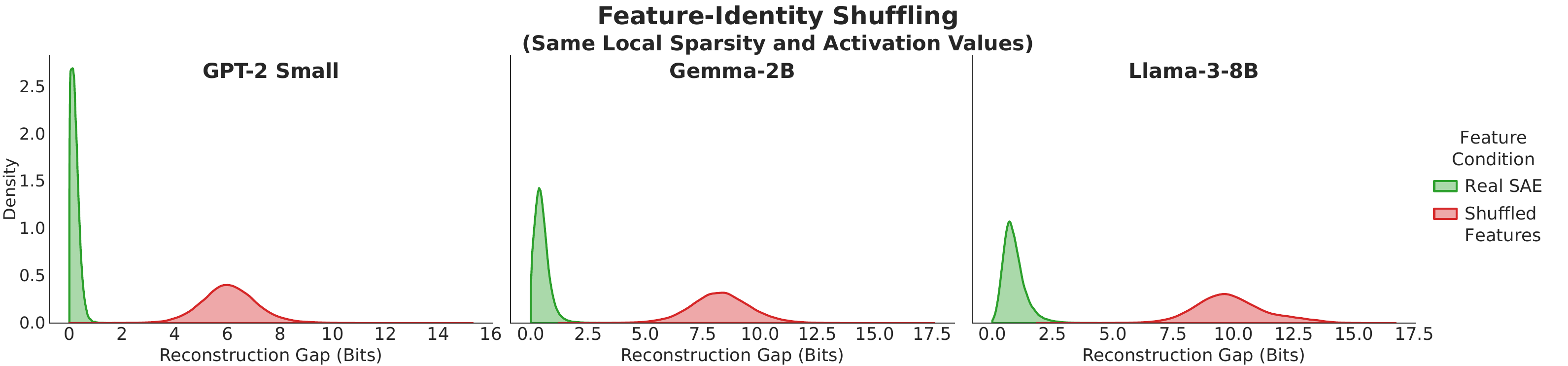}
    

    \caption{Feature-identity shuffling ablation. Shuffling substantially increases
the reconstruction-loss gap across all three model--SAE pairs.}
    \label{fig:ablation_shuffled}
\end{figure*}

\paragraph{Feature-level Logit Lens}
\label{sec:qualitative_logit_lens}

To identify \emph{what} the active SAE features represent, we verbalize them in vocabulary space. Let $d_j \in \mathbb{R}^{d_{\mathrm{model}}}$ be the decoder vector for SAE feature $j$, and $W_U \in \mathbb{R}^{d_{\mathrm{model}} \times |\mathcal{V}|}$ be the unembedding matrix. We score vocabulary items by $s_j = d_j W_U$ and report the top-scoring tokens as \emph{lexical hints} for feature $j$. These hints reveal what lexical content the feature aligns with if read directly into vocabulary space, providing a human-readable summary.


Figure \ref{fig:qualitative_late_layer_faithfulness} illustrates the qualitative shift in SAE faithfulness between early (Layer 12) and late (Layer 24) layers across two deduction prompts. At Layer 12, the SAE proxy fails to capture the base model's predictive behavior, exhibiting high KL divergence (>7.5), incorrect next-token proxy predictions (e.g., predicting `!' instead of `silent' or `insol'), and top active features whose logit-lens verbalizations (`isses', `must') are semantically disconnected from the context. In contrast, at Layer 24, the proxy accurately mirrors the base model: KL divergence collapses (dropping to 0.01 in Prompt A), and the proxy successfully approaches or matches the target predictions. Crucially, the top SAE features at this later layer become more contextually aligned, verbalizing the required logical completions (e.g., `silent, silence, quiet' and `sol, dissolve, soluble'). This qualitative pattern is consistent with the theoretical picture: later layers that yield tighter certificates also produce proxies whose outputs and feature verbalizations are more closely aligned with the base model's next-token behavior.

As an additional behavioral check, we evaluated the SAE-based proxies on three downstream tasks (Table \ref{tab:downstream_preservation}, Appendix \ref{dtp}). The same layerwise pattern persists there: layers that are easier to certify also incur substantially smaller task-performance degradation. This does not validate the theorem directly, but it shows that the certification signal tracks practical behavioral fidelity beyond the bounded-loss analysis.

\subsection{Sensitivity to Feature-Identity Shuffling}



\paragraph{Fixed global permutation protocol.}
We test whether reconstruction fidelity depends only on local sparsity
and activation values, or also on the identities of the learned SAE
features. For each model--SAE pair, we sample a single random
bijection $\pi:[m]\rightarrow[m]$ and reuse it throughout the shuffling
experiment. If coefficient $c_j(x)$ is assigned to feature $j$ in the
original Top-$k$ code, the shuffled code assigns the same coefficient
value to feature $\pi(j)$. Thus, the intervention preserves both the
number and values of active coefficients for every input.

For the calibration-derived pool $G^\star$, the corresponding shuffled
pool is $G^\star_\pi=\pi(G^\star)$. Since $\pi$ is a bijection,
\[
|G^\star_\pi|=|G^\star|=P,
\qquad
\ln\binom{m}{|G^\star_\pi|}=\ln\binom{m}{P}.
\]
Therefore, local sparsity, activation values, calibration pool size $P$,
and the exact-$P$ complexity ($C_P$) are unchanged; only the association
between coefficients and learned feature identities is disrupted.

\paragraph{Results and interpretation.}
Figure~\ref{fig:ablation_shuffled} shows that feature shuffling
substantially increases the reconstruction-loss gap for GPT-2 Small,
Gemma-2B, and Llama-3-8B. Because the original and shuffled
conditions have the same coefficient values, local support sizes,
and calibration pool size $P$, this degradation cannot be explained by
sparsity or pool size alone. It shows that the learned correspondence
between activation coefficients and SAE feature identities is
important for reconstruction fidelity. However, this does not by
itself establish semantic interpretability: a learned but
non-interpretable dictionary could also be sensitive to the same
permutation.

\section{Conclusion}
\label{sec:conclusion}

We presented a post-hoc certification framework that upper-bounds the risk of a frozen LM through an SAE-induced proxy. The resulting decomposition is useful not only because it can become non-vacuous in practical settings, but also because it localizes the sources of certification failure across proxy risk, reconstruction-loss gap, support mismatch, and optionally, an exact-P feature-pool complexity. Empirically, we observed that certification is strongly layer-dependent, and the later layers in Llama-3-8B are substantially easier to certify. We view this as a distribution-dependent, risk-level criterion for when an SAE-based proxy can be used as a cautious interpretive lens, not as a complete notion of explanation quality. More broadly, the framework provides a practical way to diagnose when SAE proxies succeed and when they should not be over-interpreted.


\section*{Acknowledgements}

Dibyanayan Bandyopadhyay acknowledges support from the Prime Minister’s Research Fellowship (PMRF) under Grant No. 2703564. The authors thank the anonymous reviewers and the Action Editor for their constructive comments and valuable feedback, which helped improve the manuscript.

\bibliography{tacl2021}
\bibliographystyle{acl_natbib}



\appendix

\clearpage

\section{Useful bound on $B$ and $\Delta$} \label{blf}

Recall

$$
B:=\log_2\!\frac{V}{\alpha},
\qquad
\Delta:=\log_2\!\left(1+\frac{(1-\alpha)V}{\alpha}\right).
$$

Then

$$
B-\Delta
=\log_2\!\frac{V}{\alpha+(1-\alpha)V}\ge0,
$$

since

$$
V\ge\alpha+(1-\alpha)V
\iff \alpha V\ge\alpha
\iff V\ge1.
$$

Hence \(B\ge\Delta\). If \(l_A,l_B\in[B-\Delta,B]\), then they lie in an interval of length \(\Delta\), and therefore

$$
|l_A-l_B|\le\Delta\le B.
$$

Thus,

$$
|l_A-l_B|\in[0,\Delta]\subseteq[0,B].
$$

\section{Auxiliary lemmas}
\label{sec:lemmas}

Throughout this section, recall that under the smoothing rule
\[
    \widetilde p(\cdot\mid x)
    =
    (1-\alpha)p(\cdot\mid x)+\alpha/V,
\]
the smoothed loss of every predictor lies in $[B-\Delta,B]$, where,
\[
    B=\log_2(V/\alpha),
    \quad
    \Delta=
    \log_2\left(
        1+\frac{(1-\alpha)V}{\alpha}
    \right)
\]
Hence the width of the loss interval is $\Delta$. In particular, for
any two predictors $f$ and $g$ with input $z$,
\[
    |\ell(f;z)-\ell(g;z)|\leq\Delta.
\]

All logarithms appearing inside concentration and complexity terms
below are natural logarithms.

\begin{lemma}[Decomposition of risk via loss gap]
\label{lem:decomp}
For any two predictors $f,g$,
\begin{equation}
R(f)
\leq
R(g)
+
\mathbb{E}_{z\sim\mathcal{D}}
\big|
\ell(f;z)-\ell(g;z)
\big|.
\end{equation}
In particular,
\begin{equation}
R(M)
\leq
R(S\circ M)
+
\epsilon_{\mathrm{loss}}.
\end{equation}
\end{lemma}

\begin{proof}
For every $z$, using triangle inequality,
\[
    \ell(f;z)
    \leq
    \ell(g;z)
    +
    |\ell(f;z)-\ell(g;z)|.
\]
Taking expectation over $z\sim\mathcal{D}$ proves the first claim.
Setting $f=M$ and $g=S\circ M$ yields the second.
\end{proof}

\begin{lemma}[Pool induced loss gap bound]
\label{lem:pool_mismatch}
Define
\begin{equation}
\epsilon_{\mathrm{pool}}
:=
\mathbb{E}_{z\sim\mathcal{D}}
\big|
\ell(S\circ M;z)
-
\ell(h_{G^\star};z)
\big|.
\end{equation}
If
\[
    \Pr_{x\sim\mathcal{D}}
    \big(E_{G^\star}(x)^c\big)
    \leq \eta,
\]
then
\begin{equation}
\epsilon_{\mathrm{pool}}
\leq
\eta\Delta.
\end{equation}
Consequently,
\begin{equation}
R(S\circ M)
\leq
R(h_{G^\star})
+
\eta\Delta.
\end{equation}
\end{lemma}

\begin{proof}
If the event $E_{G^\star}(x)$ holds, masking outside $G^\star$ does not alter
the Top-$k$ support, so $h_{G^\star}(x)=(S\circ M)(x)$.
Otherwise, since both smoothed losses lie in $[B-\Delta,B]$, their
difference is at most $\Delta$. Hence
\[
|\ell(S\circ M;z)-\ell(h_{G^\star};z)|
\leq
\Delta\,\mathbf{1}\{E_{G^\star}(x)^c\}.
\]
Taking expectation yields
$\epsilon_{\mathrm{pool}}\leq\eta\Delta$, and
Lemma~\ref{lem:decomp} gives
$R(S\circ M)\leq R(h_{G^\star})+\eta\Delta$.
\end{proof}


\begin{lemma}[Concentration of the support-mismatch rate]
\label{lem:eta_hoeffding}
Conditional on the calibration stream, fix $G^\star$ and let
\[
I_i:=\mathbf{1}\{E_{G^\star}(x_i)^c\}, \;
\eta:=\mathbb{E}[I_i],\;
\widehat{\eta}:=\frac1N\sum_{i=1}^N I_i.
\]
Then, for any $\delta\in(0,1)$, with probability at least $1-\delta$
over the held-out evaluation sample,
\[
\eta\leq\widehat{\eta}+\sqrt{\frac{\ln(1/\delta)}{2N}}.
\]
\end{lemma}

\begin{proof}
Conditional on calibration, $G^\star$ is fixed and the mismatch
indicators $I_i\in[0,1]$ are i.i.d. Therefore, Hoeffding's inequality gives
\[
\Pr(\eta-\widehat{\eta}>t)\leq e^{-2Nt^2}.
\]
Setting the right-hand side to $\delta$ proves the result.
\end{proof}

\begin{lemma}[Uniform convergence for finite classes]
\label{lem:occam}
Let $\mathcal{H}$ be a finite set of predictors and suppose that, for
every $h\in\mathcal{H}$ and every $z$, $\ell(h;z)\in[B-\Delta,B]$.

Then for any $\delta\in(0,1)$, with probability at least
$1-\delta$ over the evaluation sample
$S\sim\mathcal{D}^N$,
\begin{equation}
\forall h\in\mathcal{H}:\,
R(h)
\leq
\widehat R_S(h)
+
\Delta
\sqrt{
    \frac{
        \ln|\mathcal{H}|+\ln(1/\delta)
    }{2N}
}
\end{equation}
\end{lemma}

\begin{proof}
For fixed $h$, Hoeffding applies with range width $\Delta$:
\[
\Pr\!\left(R(h)-\widehat R_S(h)>t\right)
\leq
\exp\!\left(-\frac{2Nt^2}{\Delta^2}\right).
\]
A union bound over $|\mathcal H|$ predictors and solving for failure
probability $\delta$ gives the stated result.
\end{proof}

\begin{lemma}[Concentration of the reconstruction-loss gap]
\label{lem:epsloss_hoeffding}
Let $z_1,\ldots,z_N\overset{\mathrm{i.i.d.}}{\sim}\mathcal D$ and define
\[
    X_i
    :=
    \big|
        \ell(M;z_i)
        -
        \ell(S\circ M;z_i)
    \big|.
\]
Then
\[
    X_i\in[0,\Delta],
    \;
    \epsilon_{\mathrm{loss}}
    =
    \mathbb{E}[X_i],
    \;
    \widehat{\epsilon}_{\mathrm{loss}}
    =
    \frac{1}{N}\sum_{i=1}^{N}X_i.
\]
For any $\delta\in(0,1)$, with probability at least
$1-\delta$,
\begin{equation}
\left|
    \epsilon_{\mathrm{loss}}
    -
    \widehat{\epsilon}_{\mathrm{loss}}
\right|
\leq
\Delta
\sqrt{
    \frac{\ln(2/\delta)}{2N}
}.
\end{equation}
In particular,
\begin{equation}
\epsilon_{\mathrm{loss}}
\leq
\widehat{\epsilon}_{\mathrm{loss}}
+
\Delta
\sqrt{
    \frac{\ln(2/\delta)}{2N}
}.
\end{equation}
\end{lemma}

\begin{proof}
Since both smoothed losses lie in $[B-\Delta,B]$,
their absolute difference satisfies $X_i\in[0,\Delta]$.
Therefore, the two-sided Hoeffding inequality gives
\[
\Pr\!\left(
\left|
\widehat{\epsilon}_{\mathrm{loss}}
-\epsilon_{\mathrm{loss}}
\right|>t
\right)
\leq
2\exp\!\left(-\frac{2Nt^2}{\Delta^2}\right).
\]
Setting the right-hand side equal to $\delta$ and solving for $t$
gives the stated bound.
\end{proof}

\section{Experimental Setup}
\label{exp_setup}


We evaluate GPT-2 Small, Gemma-2B, and Llama-3-8B using pretrained
SAEs inserted at fixed internal layers. Experiments use English C4
sequences of length $32$. The feature pool $G^{\star}$ is constructed
from a disjoint calibration stream of $2.24$M tokens and is fixed
before evaluation, while the certificate curves are computed on a
separate held-out stream by varying the evaluation sample size $N$.

To make the evaluation comparable across model--SAE pairs, we impose
the same evaluation-time sparsity level, Top-$k$ with $k=64$, so that
each SAE reconstruction retains the same number of features per input.
This is an evaluation choice and should not be confused with the native
coding rule of each pretrained SAE checkpoint. The GPT-2 Small and
Gemma-2B checkpoints are ReLU/L1 SAEs, whereas the EleutherAI
Llama-3-8B SAE is natively Top-$k$-based with $k=192$. Accordingly,
for GPT-2 Small and Gemma-2B, we apply Top-$k$ to the feature
activations returned by the SAE-Lens encoder. For Llama-3-8B, we select
the $k$ largest raw encoder pre-activations and apply ReLU to the
selected coefficients before decoding. Thus, the same value of $k$ is
used across models while respecting the output provided by each SAE
implementation. We additionally vary $k$ in Appendix~\ref{topk-sensitivity} to check how
sensitive the certificate is to this evaluation choice.



\begin{table}[h]
\centering
\small
\begin{tabularx}{0.9\linewidth}{L{0.40\linewidth}Y}
\toprule
\textbf{Item} & \textbf{Setting} \\
\midrule
Text source & C4 (English) \\
Sequence construction & Contiguous token sequences \\
Sequence length & $32$ tokens \\
Calibration stream & Disjoint from evaluation stream \\
Calibration budget & $2.24$M tokens \\
Evaluation sample size $N$ & Varied to generate certificate curves \\
Default sparsity & $k=64$ \\
Loss & Smoothed bits-per-dimension \\
Smoothing parameter & $\alpha = 0.5$ \\
Bounded-loss constant & $B=\log_2(V/\alpha)$ \\
Confidence level & $\delta = 0.05$ \\
Pool construction & Using Eqn. \ref{empirical-pooling} \\
Pool-restricted predictor & Encode $\rightarrow$ mask outside $G^{\star}$ $\rightarrow$ Top-$k$ $\rightarrow$ decode $\rightarrow$ resume downstream forward pass \\
\bottomrule
\end{tabularx}
\caption{Global experimental settings shared across all model families.}
\label{tab:global_exp_setup}
\end{table}

The loss is the smoothed bits-per-dimension objective with smoothing parameter $\alpha=0.5$, yielding bounded-loss constant $B=\log_2(V/\alpha)$. We report certificates at confidence level $\delta=0.05$. Operationally, the pool-restricted predictor $h_{G}^{\star}$ is constructed by running the base model up to the probed layer, encoding the hidden state with the SAE, masking all features outside $G^*$, applying Top-$k$, decoding back to activation space, and resuming the forward pass with the reconstructed hidden state.

The global settings shared across all experiments are
summarized in Table~\ref{tab:global_exp_setup}, and the model-specific choices are summarized in Table~\ref{tab:model_exp_setup}. For the
main certificate curves, we use one primary publicly available SAE checkpoint per model
family. The layerwise sweeps in Section \ref{subsec:llama_depth_cert} and Appendix \ref{app:gpt2_layers} use additional layer-specific
checkpoints where needed.

\begin{table}[h]
\centering
\small
\begin{tabularx}{\linewidth}{l l X}
\toprule
\textbf{Model} & \textbf{SAE Release} & \textbf{Hookpoint} \\
\midrule

\textbf{GPT-2 Small} 
& \href{https://huggingface.co/jbloom/GPT2-Small-SAEs}{\makecell[l]{\texttt{gpt2-small-} \\ \texttt{res-jb}}} 
& L6: \texttt{resid\_pre} \\

\addlinespace[4pt]

\textbf{Gemma-2B} 
& \href{https://huggingface.co/jbloom/Gemma-2b-Residual-Stream-SAEs}{\makecell[l]{\texttt{gemma-2b-} \\ \texttt{res-jb}}} 
& L12: \texttt{resid\_post} \\

\addlinespace[4pt]

\textbf{Llama-3-8B} 
& \href{https://huggingface.co/EleutherAI/sae-llama-3-8b-32x}{\makecell[l]{\texttt{sae-llama-} \\ \texttt{...32x}}} 
& L30: \texttt{resid\_post} \\
\bottomrule
\end{tabularx}
\caption{Model-specific setup. Listed checkpoints are the primary SAEs used in our experiments. Because the model family, SAE release, patch layer, and hookpoint vary
across these configurations, the cross-configuration results in
Table~1 are descriptive comparisons of the selected model--SAE setups;
we do not use their relative crossing points to infer an effect of
model size alone.}
\label{tab:model_exp_setup}
\end{table}


\paragraph{Behavior under synthetic distribution shift.}
To test whether a certificate calibrated on structured natural language remains informative under distribution shift, we construct a single feature pool $G^\star$ from clean C4 and keep it fixed while evaluating on clean English (IID), $15\%$ random token replacement (Slight-OOD), and random-token sequences (Far-OOD). Consequently, $P$ and $\ln\binom{m}{P}$ remain fixed across conditions, so changes in certificate tightness arise from the distribution-dependent proxy-risk, reconstruction, and support-mismatch terms rather than from a changing $C_P$. The result is described later in Appendix \ref{iid-vs-noise} (Figure \ref{fig:bound_decomposition})

\paragraph{Frequency-thresholded pools}
To check whether the near-saturated pool is mainly caused by rare features, we keep the 2.24M-token calibration stream fixed. For each feature \(j\), we count

$$
n_j=\sum_{x\in\mathcal C}
\mathbf{1}\!\left[j\in\operatorname{supp}(\mathrm{TopK}(a(x)))\right],
$$

where \(n_j\) is the number of calibration positions in which feature \(j\) appears. We then define \(G_\tau=\{j:n_j\ge\tau\}\). Thus, \(\tau=1\) gives the original union pool, while increasing \(\tau\) removes features that appear less frequently. The result is described later in Appendix \ref{freq-pool} (Table \ref{tab:frequency_threshold})

\clearpage


\begin{table*}[t]
\section{Downstream Task Preservation}\label{dtp}
\centering
\scriptsize
\setlength{\tabcolsep}{4.5pt}
\begin{tabular}{lccccccccccc}
\toprule
& & & \multicolumn{3}{c}{\textbf{WinoGrande}} & \multicolumn{3}{c}{\textbf{PIQA}} & \multicolumn{3}{c}{\textbf{HellaSwag}} \\
\cmidrule(lr){4-6} \cmidrule(lr){7-9} \cmidrule(lr){10-12}
\textbf{Model} & \textbf{Layer} & \(\mathbf{P}\) & \(\mathrm{Acc}(M)\) & \(\mathrm{Acc}(h_G)\) & \(\Delta\)Acc & \(\mathrm{Acc}(M)\) & \(\mathrm{Acc}(h_G)\) & \(\Delta\)Acc & \(\mathrm{Acc}(M)\) & \(\mathrm{Acc}(h_G)\) & \(\Delta\)Acc \\
\midrule
GPT-2 Small & 6 & 24{,}510 & 0.510 & 0.517 & -0.007 & 0.614 & 0.604 & 0.010 & 0.343 & 0.333 & 0.010 \\
Gemma-2B & 12 & 16{,}078 & 0.619 & 0.555 & 0.064 & 0.776 & 0.710 & 0.066 & 0.547 & 0.455 & 0.092 \\
\midrule
Llama-3-8B & 4 & 72{,}752 & 0.705 & 0.503 & 0.202 & 0.804 & 0.487 & 0.317 & 0.634 & 0.273 & 0.361 \\
Llama-3-8B & 8 & 84{,}117 & 0.705 & 0.509 & 0.196 & 0.804 & 0.486 & 0.318 & 0.634 & 0.252 & 0.382 \\
Llama-3-8B & 12 & 88{,}634 & 0.705 & 0.526 & 0.179 & 0.804 & 0.502 & 0.302 & 0.634 & 0.306 & 0.328 \\
Llama-3-8B & 16 & 88{,}636 & 0.705 & 0.561 & 0.144 & 0.804 & 0.597 & 0.207 & 0.634 & 0.403 & 0.231 \\
Llama-3-8B & 20 & 109{,}646 & 0.705 & 0.560 & 0.145 & 0.804 & 0.635 & 0.169 & 0.634 & 0.441 & 0.193 \\
Llama-3-8B & 24 & 119{,}707 & 0.705 & 0.646 & 0.059 & 0.804 & 0.736 & 0.068 & 0.634 & 0.541 & 0.093 \\
Llama-3-8B & 28 & 126{,}350 & 0.705 & 0.592 & 0.113 & 0.804 & 0.762 & 0.042 & 0.634 & 0.562 & 0.072 \\
Llama-3-8B & 30 & 130{,}700 & 0.705 & 0.619 & 0.086 & 0.804 & 0.762 & 0.042 & 0.634 & 0.582 & 0.052 \\
\bottomrule
\end{tabular}
\caption{Zero-shot downstream-task preservation under sparse proxying.
We report the accuracy of the base model $M$, the pool-restricted proxy
$h_{G^\star}$, and the drop
$\Delta\mathrm{Acc}=\mathrm{Acc}(M)-\mathrm{Acc}(h_{G^\star})$
on WinoGrande~\citep{sakaguchi2019winograndeadversarialwinogradschema}, PIQA~\citep{bisk2019piqareasoningphysicalcommonsense}, and HellaSwag~\citep{zellers2019hellaswagmachinereallyfinish}. The layerwise trend mirrors the
certificate: early Llama layers, where the certificate is vacuous,
often yield near-random proxy accuracy, whereas later layers show
substantially smaller task degradation as the certificate becomes
non-vacuous and the certified risk decreases. From Layer 4 to Layer 30,
the accuracy drop decreases from 0.361 to 0.052 on HellaSwag, 0.317 to
0.042 on PIQA, and 0.202 to 0.086 on WinoGrande. This does not validate
the theorem directly, but provides an additional behavioral check
consistent with the certificate.}
\label{tab:downstream_preservation}
\end{table*}

\begin{figure*}[t]
\section{GPT-2 Small layerwise certificate analysis}
\label{app:gpt2_layers}
    \centering
    \includegraphics[width=\textwidth]{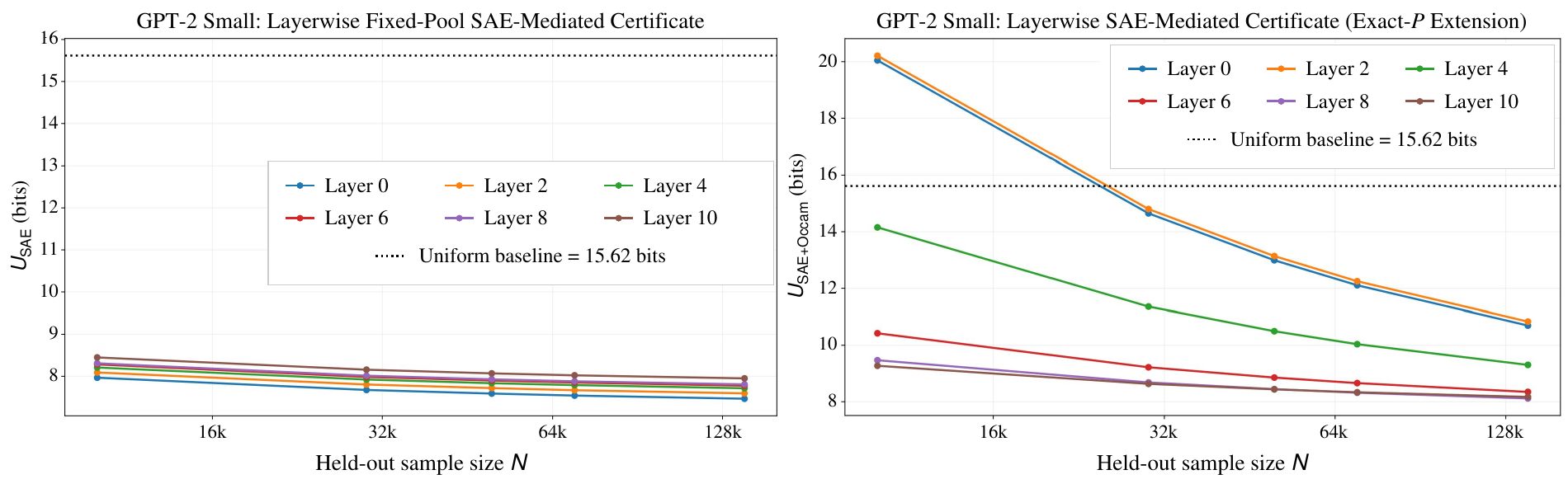}
    \caption{Layerwise certificate analysis for GPT-2 Small across layers
$\{0,2,4,6,8,10\}$. Left: fixed-pool SAE-mediated certificate
$U_{\mathrm{SAE}}$. Right: $U_{\mathrm{SAE+Occam}}$.
The fixed-pool SAE-mediated certificate shows only weak layer sensitivity, with the
$U_{\mathrm{SAE}}$ curves remaining closely grouped across depth.
In contrast, the exact-$P$ extension shows stronger layer dependence
because it additionally contains the layer-dependent complexity
$C_P=\ln\binom{m}{P}$. Since this finite-sample term decreases with
increasing $N$, the layerwise differences become smaller at larger
sample sizes.}
    \label{fig:gpt-layerwise}
\end{figure*}



\begin{figure*}[!t]
\section{Top-k sensitivity}\label{topk-sensitivity}
    \centering
    \includegraphics[width=\textwidth]{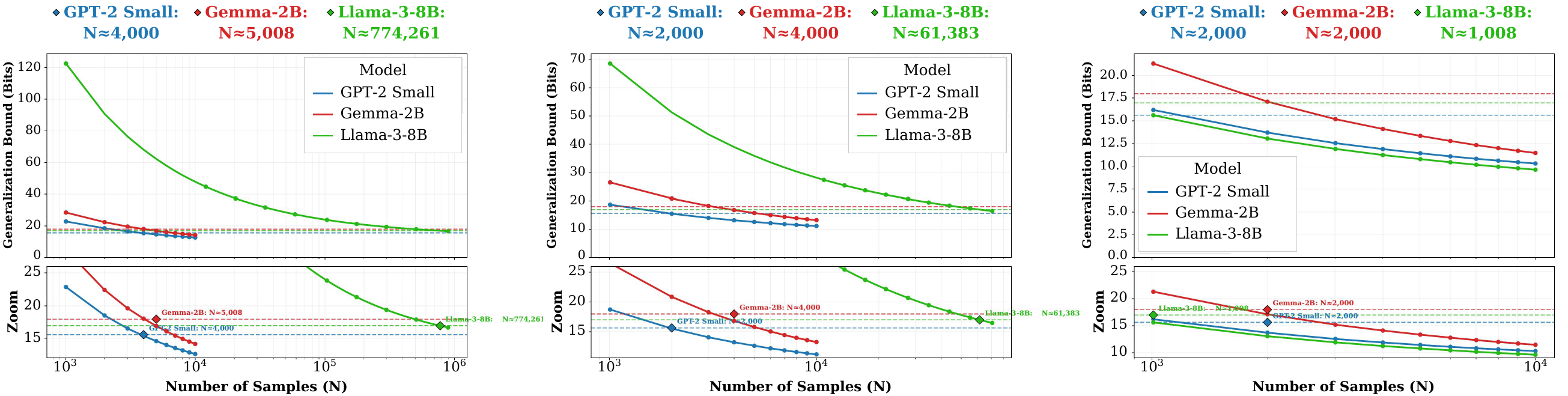}
\caption{Top-$k$ sensitivity at the representative layers.
Across $k\in\{16,32,128\}$, increasing $k$ generally reduces the sample
size required for non-vacuity. For $k=16$, the crossing points are
approximately 4,000, 5,008, and 774,261 for GPT-2 Small, Gemma-2B, and
Llama-3-8B, respectively; for $k=32$, they decrease to approximately
2,000, 4,000, and 61,383. At $k=128$, Llama-3-8B reaches non-vacuity
at approximately 1,008 samples, while GPT-2 Small and Gemma-2B cross
at approximately 2,000. Thus, less aggressive sparsification generally
makes certification easier, although the quantitative effect is
model--SAE dependent.}
    \label{fig:topk_sensitivity}
\end{figure*}

\begin{figure*}[!t]
\section{Behavior of the Certificate under Synthetic Input Corruption}\label{iid-vs-noise}

    \centering
    \includegraphics[width=\textwidth]{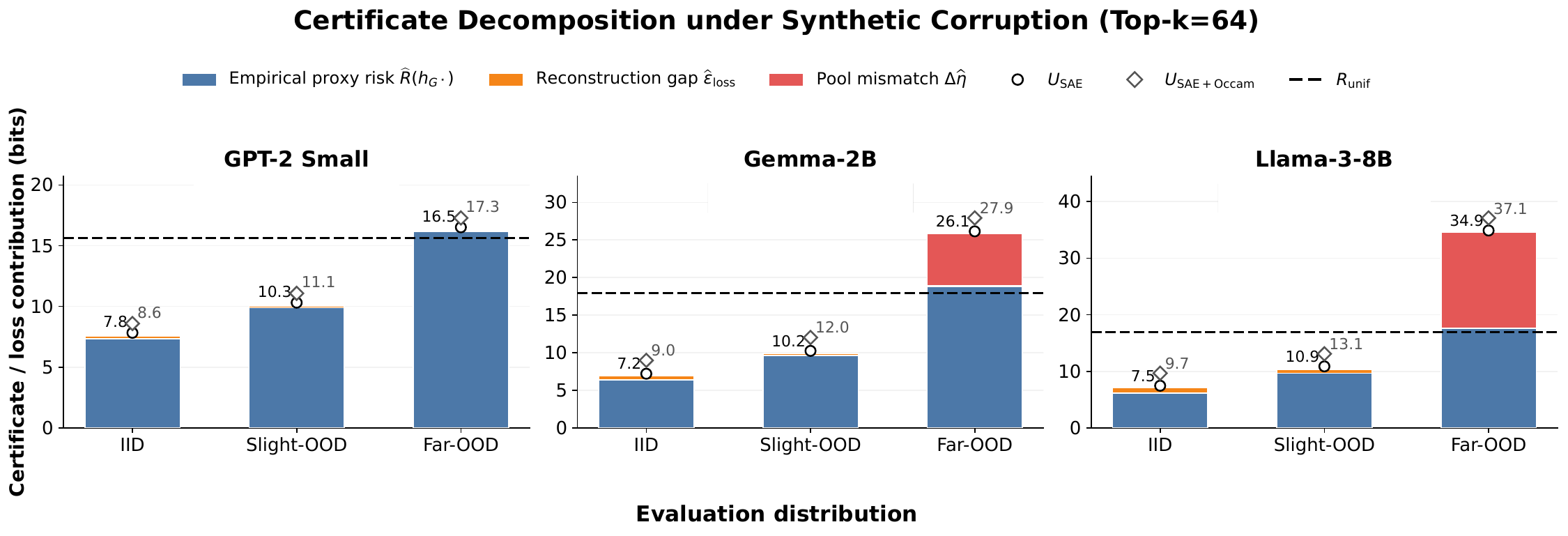}
    \caption{\textbf{Certificate decomposition under synthetic input corruptio}n (\(k=64\), \(N=70{,}000\) evaluation samples per condition). IID denotes clean English, Slight-OOD denotes 15\% random token replacement, and Far-OOD denotes random-token sequences. Each model’s clean-C4 calibration pool is fixed across conditions. Bars show the empirical proxy-risk, reconstruction-loss-gap, and support-mismatch contributions; circles and diamonds show \(U_{\mathrm{SAE}}\) and \(U_{\mathrm{SAE+Occam}}\), respectively. Dashed lines indicate \(R_{\mathrm{unif}}\). Both certificates remain non-vacuous under IID and Slight-OOD and become vacuous under Far-OOD.}

    \label{fig:bound_decomposition}
\end{figure*}

\begin{table*}[t]
\section{Frequency-thresholded pools}\label{freq-pool}
\centering
\scriptsize
\adjustbox{width=\textwidth}{
\begin{tabular}{c|ccccccc}
\hline
$\tau$
& $P/m$
& $\hat{\eta}$
& $\widehat{R}(h_{G_\tau})$
& $U_{\mathrm{SAE}}$
& $U_{\mathrm{SAE+Occam}}$
& Asymp.
& $N^\star$ \\
\hline
1  & 0.9972 & 0.0029 & 6.2064 & 7.9189  & 13.7391 & 7.1708  & 10K \\
2  & 0.9952 & 0.0112 & 6.2064 & 8.0597  & 15.3922 & 7.3116  & 10K \\
5  & 0.9888 & 0.0545 & 6.2064 & 8.7945  & 19.3223 & 8.0464  & 30K \\
10 & 0.9767 & 0.1630 & 6.2065 & 10.6356 & 24.8384 & 9.8875  & 50K \\
20 & 0.9533 & 0.4070 & 6.2065 & 14.7760 & 33.4041 & 14.0279 & 450K \\
50 & 0.8645 & 0.8853 & 6.2065 & 22.8921 & 50.0122 & 22.1440 & Never \\
\hline
\end{tabular}}
\caption{Frequency-thresholded pools for Llama-3-8B Layer~30 using the
same 2.24M-token calibration stream. A feature is retained in
$G_\tau$ only if it occurs at least $\tau$ times; $\tau=1$ recovers
the original union pool. Bounds are reported at $N=10$K. Asymp. is
$\widehat R(h_{G_\tau})+\widehat\epsilon_{\mathrm{loss}}
+\Delta\widehat\eta$, and $N^\star$ is the first point on the
$10$K, $30$K, $50$K,\ldots grid where
$U_{\mathrm{SAE+Occam}}<R_{\mathrm{unif}}$.}
\label{tab:frequency_threshold}
\end{table*}





\end{document}